\documentclass{article}

\usepackage{microtype}
\usepackage{graphicx}
\usepackage{caption}
\usepackage{subcaption}
\usepackage{scalerel}
\usepackage{booktabs} 
\usepackage{relsize}
\usepackage{mathtools}
\usepackage{tikz}
\usepackage{multirow}
\usepackage{hyperref}
\usepackage{caption}
\usepackage{threeparttable}

\usepackage{amssymb,amsmath,amsthm, amsfonts}

\usepackage{amsmath,amsfonts,bm}

\def\eqref#1{equation~\ref{#1}}
\def\Eqref#1{Equation (\ref{#1})}
\def\Eqrefs#1#2#3{Equations (\ref{#1}), (\ref{#2}), (\ref{#3})}

\def\1{\bm{1}}

\def\vv{{\bm{v}}}

\DeclareMathAlphabet{\mathsfit}{\encodingdefault}{\sfdefault}{m}{sl}
\SetMathAlphabet{\mathsfit}{bold}{\encodingdefault}{\sfdefault}{bx}{n}

\def\gC{{\mathcal{C}}}
\def\gD{{\mathcal{D}}}

\def\gF{{\mathcal{F}}}
\def\gG{{\mathcal{G}}}

\def\gL{{\mathcal{L}}}

\def\gR{{\mathcal{R}}}
\def\gS{{\mathcal{S}}}
\def\gT{{\mathcal{T}}}

\def\gX{{\mathcal{X}}}

\newcommand{\E}{\mathbb{E}}

\newcommand{\R}{\mathbb{R}}

\DeclareMathOperator*{\argmin}{arg\,min}

\newtheorem{assumption}{Assumption}
\newtheorem{theorem}{Theorem}[section]
\newtheorem{corollary}{Corollary}[theorem]
\newtheorem{lemma}[theorem]{Lemma}
\newtheorem{prop}[theorem]{Proposition}

\newtheorem{definition}[theorem]{Definition}
\newtheorem*{remark}{Remark}

\usepackage[accepted]{arxiv}

\icmltitlerunning{Contrastive Unsupervised Representation Learning}

\begin{document}

\twocolumn[
\icmltitle{A Theoretical Analysis of Contrastive Unsupervised Representation Learning}



\icmlsetsymbol{equal}{*}

\begin{icmlauthorlist}
\icmlauthor{Sanjeev Arora}{pr,ias}
\icmlauthor{Hrishikesh Khandeparkar}{pr}
\icmlauthor{Mikhail Khodak}{cm} 
\icmlauthor{Orestis Plevrakis}{pr}
\icmlauthor{Nikunj Saunshi}{pr}
\vspace{0.02in}
\end{icmlauthorlist}
\begin{emaillist}
\email{\normalfont\texttt{\{arora,\,hrk,\,orestisp,\,nsaunshi\}@cs.princeton.edu}}
\email{\normalfont\texttt{\qquad khodak@cmu.edu}}
\end{emaillist}

\icmlaffiliation{pr}{Princeton University, Princeton, New Jersey, USA.}
\icmlaffiliation{ias}{Institute for Advanced Study, Princeton, New Jersey, USA.}
\icmlaffiliation{cm}{Carnegie Mellon University, Pittsburgh, Pennsylvania, USA}


\icmlkeywords{Machine Learning, ICML}

\vskip 0.3in
]



\printAffiliationsAndNotice{}  

\begin{abstract}
Recent empirical works have successfully used unlabeled data to learn feature representations that are broadly useful in downstream classification tasks.
Several of these methods are reminiscent of the well-known word2vec embedding algorithm: leveraging availability of pairs of semantically ``similar"  data points and ``negative samples," the learner forces the inner product of representations of similar pairs with each other to be higher on average than with negative samples.
The current paper uses the term {\em contrastive learning} for such algorithms and presents a theoretical framework for analyzing them by introducing {\em latent classes} and hypothesizing that semantically similar points are sampled from the same latent class. 
This framework allows us to show provable guarantees on the performance of the learned representations on the average classification task that is comprised of a subset of the same set of latent classes.
Our generalization bound also shows that learned representations can reduce (labeled) sample complexity on downstream tasks.
We conduct controlled experiments in both the text and image domains to support the theory.
\end{abstract}

\section{Introduction}\label{sec:introduction}
This paper concerns {\em unsupervised representation learning}: using unlabeled data to learn a representation function $f$ such that replacing data point $x$ by feature vector $f(x)$ in new classification tasks reduces the requirement for labeled data.
This is distinct from {\em semi-supervised learning}, where learning can leverage unlabeled as well as labeled data. (Section \ref{sec:related} surveys other prior ideas and models).

For images, a  {\em proof of existence} for broadly useful representations is the output of the penultimate layer (the one before the softmax) of a powerful deep net trained on ImageNet. 
In natural language processing (NLP), low-dimensional representations of text -- called {\em text embeddings} -- have been computed with unlabeled data \cite{Peters:18,Devlin:18}.
Often the embedding function is trained by using the embedding of a piece of text to predict the surrounding text \cite{Kiros:15,Logeswaran:18,Pagliardini:18}.
Similar methods that leverage similarity in nearby frames in a video clip have had some success for images as well \cite{Wang:15}.

Many of these algorithms are related: they assume access to pairs or tuples (in the form of co-occurrences) of text/images that are more {\em semantically similar} than randomly sampled text/images, and their objective forces representations to respect this similarity on average. 
For instance, in order to learn a representation function $f$ for sentences, a simplified version of what \citet{Logeswaran:18} minimize is the following loss function
\begin{equation*}\label{eqn:w2v1}
\mathop{\E}\limits_{x, x^+, x^{-}}\left[-\log\left(\frac{e^{f\left(x\right)^Tf\left(x^+\right)}}{e^{f\left(x\right)^Tf\left(x^+\right)}+e^{f\left(x\right)^Tf\left(x^-\right)}}\right)\right]
\end{equation*} 
where $(x, x^{+})$ are a similar pair and $x^{-}$ is presumably dissimilar to $x$ (often chosen to be a random point) and typically referred to as a {\em negative sample}.
Though reminiscent of past ideas -- e.g. kernel learning, metric learning, co-training \cite{Cortes:10,Bellet:13,Blum:98} -- these algorithms lack a theoretical framework {\em quantifying} when and why they work.
While it seems intuitive that minimizing such loss functions should lead to representations that capture `similarity,' formally it is unclear 
why the learned representations should do well on downstream {\em linear classification tasks} -- their somewhat mysterious success is often treated as an obvious consequence. 
To analyze this success, a framework must connect  `similarity' in unlabeled data with the semantic information that is implicitly present in downstream tasks. 

We propose the term {\em Contrastive Learning} for such methods and provide a new conceptual framework with minimal assumptions\footnote{The alternative would be to make assumptions about generative models of data.
This is difficult for images and text. 
}.
Our main contributions are the following:
\begin{enumerate}
\item We formalize the notion of semantic similarity by introducing {\em latent classes}. Similar pairs are assumed to be drawn from the same latent class.
A downstream task is comprised of a subset of these latent classes.
\item Under this formalization, we prove that a representation function $f$ learned from a function class $\gF$ by contrastive learning has low {\em average} linear classification loss if $\gF$ contains a function with low unsupervised loss.
Additionally, we show a generalization bound for contrastive learning that depends on the Rademacher complexity of $\gF$.
After highlighting inherent limitations of negative sampling, we show sufficient properties of $\gF$ which allow us to overcome these limitations.
\item Using insights from the above framework, we provide a novel extension of the algorithm that can leverage larger blocks of similar points than pairs, has better theoretical guarantees, and performs better in practice.
\vspace*{-0.25in}
%
\end{enumerate}
Ideally, one would like to show that contrastive learning always gives representations that {\em compete} with those learned from the same function class with plentiful labeled data.
Our formal framework allows a rigorous study of such questions: we show a simple counterexample that prevents such a blanket statement without further assumptions.
However, if the representations are well-concentrated and the mean classifier (Definition~\ref{def:avg_sup_mean}) has good performance, we can show a weaker version of the ideal result (Corollary \ref{corr:subgaussian}).
Sections~\ref{sec:framework} and \ref{sec:competitive} give an overview of the framework and the results, and subsequent sections deal with the analysis. Related work is discussed in Section~\ref{sec:related} and Section~\ref{sec:experiment} describes experimental verification and support for our framework. 

\section{Framework for Contrastive Learning}\label{sec:framework}
We first set up notation and describe the framework for unlabeled data and classification tasks that will be essential for our analysis.
Let $\gX$ denote the set of all possible data points.
Contrastive learning assumes access to {\em similar} data in the form of pairs $(x, x^+)$ that come from a distribution  $\gD_{sim}$ as well as $k$ i.i.d. {\em negative samples} $x^-_1, x^-_2, \dots, x^-_k$ from a distribution  $\gD_{neg}$ that are presumably unrelated to $x$. 
Learning is done over $\gF$, a class of {\em representation functions} $f: \gX \rightarrow \mathbb{R}^d$, such that $\|f(\cdot)\|\leq R$ for some $R>0$. 

\subsection*{Latent Classes}
To formalize the notion of semantically similar pairs $(x, x^+)$, we introduce the concept of {\em latent classes}.

{\em Let $\gC$ denote the set of all latent classes.
Associated with each class $c\in\gC$ is a probability distribution  $\gD_c$ over $\gX$
.}

%
Roughly, $\gD_c(x)$ captures how relevant $x$ is to class $c$. For example, $\gX$ could be natural images and $c$ the class ``dog"  whose associated $\gD_c$ assigns high probability to images containing dogs and low/zero probabilities to other images.
Classes can overlap arbitrarily.\footnote{An image of a dog by a tree can appear in both $\gD_{dog}$ \&  $\gD_{tree}$.} Finally, we assume a distribution $\rho$ over the classes that characterizes how these classes naturally occur in the unlabeled data.
Note that we make no assumption about the functional form of $\gD_c$ or $\rho$.
\subsection*{Semantic Similarity}
To formalize similarity, we assume similar data points $x, x^+$ are i.i.d. draws from the same class distribution  
$\gD_c$ for some class $c$
picked randomly according to measure $\rho$. Negative samples are drawn from the marginal of $\gD_{sim}$:
\begin{align}
\gD_{sim}(x,x^+)&=\mathop{\E}\limits_{c\sim\rho}\gD_c(x)\gD_c(x^+)\label{eq:sim_distribution}\\
\gD_{neg}(x^-)&=\mathop{\E}\limits_{c\sim\rho}\gD_c(x^-)\label{eq:neg_distribution}
\end{align}
Since classes are allowed to overlap and/or be fine-grained, this is a plausible formalization of ``similarity."
As the identity of the class in not revealed, we call it unlabeled data.
Currently empirical works heuristically identify such similar pairs from co-occurring image or text data.

\subsection*{Supervised Tasks} \label{subsec:sup_tasks}

We now characterize the tasks that a representation function $f$ will be tested on.
A $(k +1)$-way\footnote{We use $k$ as the number of negative samples later.} supervised task $\gT$ consists of distinct classes $\{c_1,\dots,c_{k+1}\}\subseteq \gC$.
The labeled dataset for the task $\gT$ consists of $m$ i.i.d. draws from the following process: 

{\em A label $c \in \{c_1,...,c_{k+1}\}$ is picked according to a distribution $\gD_\gT$.
Then, a sample $x$ is drawn from $\gD_c$.
Together they form a labeled pair $(x,c)$ with distribution} 
\begin{equation}\label{eq:sup_distribution}
\gD_{\gT}(x,c)=\gD_c(x)\gD_\gT(c)
\end{equation}
A key subtlety in this formulation is that the classes in downstream tasks and their associated data distributions $\gD_c$ are the same as in the unlabeled data.
This provides a path to formalizing how capturing similarity in unlabeled data can lead to quantitative guarantees on downstream tasks.
$\gD_\gT$ is assumed to be uniform\footnote{We state and prove the general case in the Appendix.} for theorems in the main paper.

\subsection*{Evaluation Metric for Representations}
The quality of the representation function $f$ is evaluated by its performance on a 
multi-class classification task $\gT$ using {\em linear classification}.
For this subsection, we fix a task $\gT =\{c_1,...,c_{k+1}\}$.
A multi-class classifier for $\gT$ is a function $g:\gX\rightarrow\R^{k+1}$ whose output coordinates are indexed by the classes $c$ in task $\gT$ 
.

The loss incurred by $g$ on point $(x,y)\in\gX\times\gT$ is defined as $\ell(\{g(x)_y-g(x)_{y'}\}_{y' \neq y})$, which is a function of a $k$-dimensional vector of differences in the coordinates.
The two losses we will consider in this work are the standard hinge loss $\ell(\vv)=\max\{0,1+\max_i\{-\vv_i\}\}$ and the logistic loss $\ell(\vv) = \log_2\left(1+\sum_{i}\exp(-\vv_i)\right)$ for $\vv\in\R^k$.
Then the supervised loss of the classifier $g$ is
\begin{equation*}
L_{sup}(\gT,g)\coloneqq\mathop{\E}_{(x,c)\sim \gD_{\gT}} \big[\ell \big(\{g(x)_c-g(x)_{c'}\}_{c'\neq c}\big)\big]
\end{equation*}
To use a representation function $f$ with a linear classifier, a matrix $W\in\R^{(k+1)\times d}$ is trained and $g(x)=Wf(x)$ is used to evaluate classification loss on tasks.
Since the best $W$ can be found by fixing $f$ and training a linear classifier, we abuse notation and define the {\em supervised loss} of $f$ on $\gT$ to be the loss when the best $W$ is chosen for $f$:
\begin{equation}\label{eqn:suploss}
L_{sup}(\gT,f)=\inf_{W\in \mathbb{R}^{(k+1)\times d}}L_{sup}(\gT, Wf)
\end{equation}
Crucial to our results and experiments will be a specific $W$ where the rows are the means of the representations of each class which we define below.
\begin{definition}[Mean Classifier]\label{def:avg_sup_mean}
For a function $f$ and task $\gT = (c_1, \dots, c_{k+1})$, the mean classifier is $W^\mu$ whose $c^{th}$ row is the mean $\mu_c$ of representations of inputs with label $c$: $\mu_c\coloneqq\mathop{\E}\limits_{x\sim\gD_c}[f(x)]$. We use $L^{\mu}_{sup}(\gT,f) \coloneqq L_{sup}(\gT,W^\mu f)$ as shorthand for its loss.
\end{definition}
Since contrastive learning has access to data with latent class distribution $\rho$, it is natural to have better guarantees for tasks involving classes that have higher probability in $\rho$.

\begin{definition}[Average Supervised Loss]\label{def:sup_loss}
Average loss for a function $f$ on $(k+1)$-way tasks is defined as
\begin{align*}
L_{sup}(f)\coloneqq\mathop{\E}\limits_{\{c_i\}_{i=1}^{k+1}\sim\rho^{k+1}}\left[L_{sup}(\{c_i\}_{i=1}^{k+1},f)\ |\ c_i\neq c_j\right]
\end{align*}
The average supervised loss of its {\em mean classifier} is
\begin{align*}
L_{sup}^\mu(f)\coloneqq\mathop{\E}\limits_{\{c_i\}_{i=1}^{k+1}\sim\rho^{k+1}}\left[L_{sup}^\mu(\{c_i\}_{i=1}^{k+1},f)\ |\ c_i\neq c_j\right]
\end{align*}
\end{definition}
\vspace*{-0.1in}


\subsection*{Contrastive Learning Algorithm}
We describe the training objective for contrastive learning: the choice of loss function is dictated by the $\ell$ used in the supervised evaluation, and $k$ denotes number of negative samples used for training.
 Let $(x,x^+) \sim \gD_{sim}$, $(x^-_1, .., x_k^-) \sim \gD_{neg}^k$ as defined in Equations~(\ref{eq:sim_distribution}) and~(\ref{eq:neg_distribution}).
\begin{definition}[Unsupervised Loss]\label{def:unsuploss}
The population loss is
\begin{equation}\label{eq:QT}
L_{un}(f)\coloneqq\mathop{\E}\left[ \ell \left(\left\{f(x)^T\big(f(x^+)-f(x^-_i)\right)\right\}_{i=1}^k \right) \Big]
\end{equation}
and its empirical counterpart with M samples $(x_j, x^+_j, x_{j1}^-,...,x_{jk}^-)_{j=1}^M$ from $\gD_{sim}\times\gD_{neg}^k$ is 
 \begin{equation}\label{def:emp_unsuploss}
 \widehat{L}_{un}(f)=\frac{1}{M}\sum_{j=1}^M  \ell \left(\left\{f(x_{j})^T\left(f(x_{j}^+)-f(x_{ji}^-)\right)\right\}_{i=1}^k \right)
\end{equation}
\end{definition}
Note that, by the assumptions of the framework described above, we can now express the unsupervised loss as
\begin{align*}
&L_{un}(f) \\
&\quad=\mathop{\E}\limits_{\substack{c^+,c_i^- \\\sim \rho^{k+1}}}\mathop{\E}\limits_{\substack{x,x^+\sim \gD_{c^+}^2 \\ x_i^- \sim  \gD_{c_i^-}}} \left[\ell \left( \left\{f(x)^T\left(f(x^+)  -  f(x_i^-)  \right)\right\} \right) \right]
\end{align*}
The algorithm to learn a representation function from $\gF$ is to find a function $\widehat{f} \in \argmin_{f \in \gF}{\widehat L_{un}(f)}$ that minimizes the empirical unsupervised loss.
This function $\widehat f$ can be subsequently used for supervised linear classification tasks.
In the following section we proceed to give an overview of our results that stem from this framework.

\section{Overview of Analysis and Results} \label{sec:competitive}
What can one {\em provably} say about the performance of $\widehat f$?
As a first step we show that $L_{un}$ is like a ``surrogate" for $L_{sup}$ by showing that $L_{sup}(f)\le\alpha L_{un}(f),\forall f\in\gF$, suggesting that minimizing $L_{un}$ makes sense.
This lets us show a bound on the supervised performance $L_{sup}(\widehat f)$ of the representation learned by the algorithm.
For instance, when training with one negative sample, the performance on average binary classification has the following guarantee:


{\bf Theorem {\color{red} \ref{thm:unsup_upper_bound} }}{\em (Informal binary version).
 \begin{align*}
&L_{sup}(\widehat f) \leq \alpha L_{un}(f) + \eta\ Gen_M + \delta &\forall f \in \gF
\end{align*}
where $\alpha, \eta,\delta$ are constants depending on the distribution $\rho$ and $Gen_M \to 0$ as $M \to \infty$. When $\rho$ is uniform and $|\gC| \to \infty$, we have that $\alpha, \eta \to 1,\ \delta \to 0$}. 

At first glance, this bound seems to offer a somewhat complete picture: {\em When the number of classes is large, if the unsupervised loss can be made small by $\gF$, then the supervised loss of $\widehat{f}$, learned using finite samples, is small.}

While encouraging, this result still leaves open the question: Can $L_{un}(f)$ indeed be made small on reasonable datasets using function classes $\gF$ of interest, even though the similar pair and negative sample can come from the same latent class?
We shed light on this by upper-bounding $L_{un}(f)$ by two components: 
(a) the loss $L_{un}^{\neq}(f)$ for the case where the positive and negative samples are from different classes;
(b) a notion of deviation $s(f)$, within each class.

{\bf Theorem \ref{thm:binary_theorem}} {\em (Informal binary version).
\begin{align*}
&L_{sup}(\widehat f) \leq L_{un}^{\neq}(f) + \beta s(f) + \eta\ Gen_M &&\forall f \in \gF
\end{align*}
for constants $\beta,\eta$ that depend on the distribution $\rho$. Again, when $\rho$ is uniform and $|\gC| \to \infty$ we have $\beta \to 0, \eta \to 1$}.

This bound lets us infer the following: {\em if the class $\gF$ is rich enough to contain a function $f $ for which $L_{un}^{\neq}(f)+\beta s(f)$ is low, then $\widehat f$ has high supervised performance.}
Both $L_{un}^{\neq}(f)$ and $s(f)$ can potentially be made small for rich enough $\gF$.

Ideally, however, one would want to show that $\widehat f$ can compete on classification tasks with every $f \in \gF$
\begin{equation}\label{ml_dream}
\text{(\em Ideal Result):}\quad L_{sup}(\widehat f) \leq \alpha L_{sup}(f) + \eta\ Gen_M
\end{equation}
Unfortunately, we show in Section~\ref{subsec:counter} that the algorithm can pick something far from the optimal $f$.
However, we extend Theorem~\ref{thm:binary_theorem} to a bound similar to (\ref{ml_dream}) (where the classification is done using the mean classifier) under assumptions about the intraclass concentration of $f$ and about its mean classifier having high margin. 

Sections \ref{subsec:k-way_guarantees} and \ref{subsec:k-way_effect} extend our results to the more complicated setting where the algorithm uses $k$ negative samples (\ref{eq:QT}) and note an interesting behavior: increasing the number of negative samples beyond a threshold can hurt the performance.
In Section \ref{subsec:CURL} we show a novel extension of the algorithm that utilizes larger blocks of similar points.
Finally, we perform controlled experiments in Section \ref{sec:experiment} to validate components of our framework and corroborate our suspicion that the mean classifier of representations learned using labeled data has good classification performance.

\section{Guaranteed Average Binary Classification}\label{sec:power_of_framework}

To provide the main insights, we prove the algorithm's guarantee when we use only 1 negative sample ($k=1$).
For this section, let $L_{sup}(f)$ and $L^{\mu}_{sup}(f)$ be as in Definition \ref{def:sup_loss} for binary tasks. 
We will refer to the two classes in the supervised task as well as the unsupervised loss as $c^+, c^-$.
Let $\gS=\{x_j,x_j^+,x_j^-\}_{j=1}^M$ be our training set sampled from the distribution $\gD_{sim}\times\gD_{neg}$ and $\widehat f\in\argmin_{f\in\gF}\widehat L_{un}(f)$.

\subsection{Upper Bound using Unsupervised Loss}\label{subsec:upper_bound}
Let $f_{|\gS}=\left(f_t(x_j),f_t(x_j^+),f_t(x_j^-)\right)_{\substack{j\in[M], t \in [d]}} \in \mathbb{R}^{3dM}$  be the restriction on $\gS$ for any $f\in \gF$.
Then, the statistical complexity measure relevant to the estimation of the representations is the following Rademacher average
\begin{align*}\label{eq:rademacher}
\gR_\gS(\gF)=\mathop{\E}\limits_{\sigma \sim \{\pm1\}^{3dM}} \big[ \sup_{f\in \gF} \langle \sigma, f_{|\gS} \rangle \big]
\end{align*}
Let $\tau=\mathop{\E}\limits_{c,c'\sim\rho^2}\1\{c=c'\}$ be the probability that two classes sampled independently from $\rho$ are the same.
\begin{theorem}\label{thm:unsup_upper_bound}
With probability at least $1-\delta$, for all $f\in\gF$
\begin{equation*}
L^{\mu}_{sup}(\widehat f)\le\frac{1}{(1-\tau)}(L_{un}(f)-\tau) + \frac{1}{(1-\tau)}Gen_M
\end{equation*}
where
\begin{equation*}\label{eq:gen_bound}
Gen_M=O \left(R\frac{   \mathcal{R}_\gS(\gF)}{M}+  R^2\sqrt{\frac{\log{\frac{1}{\delta}}}{M}} \right)
\end{equation*}
\end{theorem}
\begin{remark}
The complexity measure $\gR_S(\gF)$ is tightly related to the labeled sample complexity of the classification tasks.
For the function class $\gG=\{w^T f(\cdot) | f\in \gF,\ \|w\|\leq 1\}$ that one would use to solve a binary task from scratch using labeled data, it can be shown that $\mathcal{R}_{\gS}(\gF) \leq d \mathcal{R}_\gS(\gG)$, where $\mathcal{R}_\gS(\gG)$ is the usual Rademacher complexity of $\gG$ on $\gS$ (Definition 3.1 from \cite{Mohri:18}).
\end{remark}
We state two key lemmas needed to prove the theorem.
\begin{lemma}\label{lemma:gen_bound}
With probability at least $1-\delta$ over the training set $\gS$, for all $f\in \gF$
\begin{align*}
L_{un}(\widehat{f})\leq  L_{un}(f)+Gen_M
\end{align*}
\end{lemma}
\vspace*{-0.1in}

We prove Lemma \ref{lemma:gen_bound} in Appendix {\color{red} \ref{appdx:gen_bound}}.

\begin{lemma}\label{lemma:jensen}
For all $f\in\gF$
\begin{equation*}
L_{sup}^{\mu}(f)\le \frac{1}{(1-\tau)}\left(L_{un}(f)-\tau \right)
\end{equation*}
\begin{proof}
The key idea in the proof is the use of Jensen's inequality.
Unlike the unsupervised loss which uses a random point from a class as a classifier, using the mean of the class as the classifier should only make the loss lower.
Let $\mu_c=\mathop{\E}\limits_{x\sim \gD_c}f(x)$ be the mean of the class $c$. 
\begin{align*}
&L_{un}(f) = \mathop{\E}\limits_{\substack{(x,x^{+})\sim \gD_{sim}\\x^{-}\sim \gD_{neg}}} \left[\ell(f(x)^T(f(x^{+})-f(x^{-})))\right]\\
&=^{(a)} \mathop{\E}\limits_{\substack{c^{+},c^{-}\sim \rho^2\\x\sim \gD_{c^{+}}}}\mathop{\E}\limits_{\substack{x^{+}\sim \gD_{c^{+}}\\x^{-}\sim \gD_{c^{-}}}} \left[\ell(f(x)^T(f(x^{+})-f(x^{-})))\right]\\
&\ge^{(b)} \mathop{\E}\limits_{c^{+},c^{-}\sim \rho^2}\mathop{\E}\limits_{x\sim \gD_{c^{+}}} \left[\ell(f(x)^T(\mu_{c^{+}}-\mu_{c^{-}}))\right]\\
&=^{(c)}(1-\tau)\mathop{\E}\limits_{c^{+},c^{-}\sim \rho^2} [L^{\mu}_{sup}(\{c^{+},c^{-}\}, f)|c^{+}\neq c^{-}]+\tau \\
&=^{(d)} (1-\tau)L_{sup}^{\mu}(f)+\tau
\end{align*}
where (a) follows from the definitions in (\ref{eq:sim_distribution}) and (\ref{eq:neg_distribution}), (b) follows from the convexity of $\ell$ and Jensen's inequality by taking the expectation over $x^{+}$, $x^{-}$ inside the function, (c) follows by splitting the expectation into the cases $c^{+}=c^{-}$ and $c^{+}\neq c^{-}$, from symmetry in $c^{+}$ and $c^{-}$ in sampling and since classes in tasks are uniformly distributed (general distributions are handled in Appendix~\ref{appdx:multiclass}). Rearranging terms completes the proof.
\end{proof}
\end{lemma}
\begin{proof}[Proof of Theorem \ref{thm:unsup_upper_bound}]
The result follows directly by applying Lemma \ref{lemma:jensen} for $\widehat f$ and finishing up with Lemma \ref{lemma:gen_bound}.
\end{proof}

One could argue that if $\gF$ is rich enough such that $L_{un}$ can be made small, then Theorem~\ref{thm:unsup_upper_bound} suffices.
However, in the next section we explain that unless $\tau\ll 1$, this may not always be possible and we show one way to alleviate this.

\subsection{Price of Negative Sampling: Class Collision}\label{subsec:collision}
Note first that the unsupervised loss can be decomposed as
\begin{equation}\label{eq:decomposition}
L_{un}(f) = \tau L^=_{un}(f)+(1-\tau) L^{\neq}_{un}(f)
\end{equation}
where $L^{\neq}_{un}(f)$ is the loss suffered when the similar pair and the negative sample come from different classes. 
\begin{align*}
&L^{\neq}_{un}(f)\\
&\quad= \mathop{\E}\limits_{\substack{c^{+},c^{-}\sim\rho^2\\x,x^{+}\sim\gD_{c^{+}}^2\\x^{-}\sim\gD_{c^{-}}}} \left[\ell(f(x)^T(f(x^{+})-f(x^{-}))) | c^{+}\neq c^{-}\right]
\end{align*}
and $L^=_{un}(f)$ is when they come from the {\em same class}. Let $\nu$ be a distribution over $\gC$ with $\nu(c)\propto \rho^2(c)$, then
\begin{align*}
L^{=}_{un}(f) &= \mathop{\E}\limits_{\substack{c \sim \nu \\ x,x^{+},x^{-} \sim \gD_{c}^3}} \left[\ell(f(x)^T(f(x^{+})-f(x^{-})))\right]\\
	&\ge \mathop{\E}\limits_{c\sim \nu, x\sim\gD_{c}} \left[\ell(f(x)^T(\mu_{c}-\mu_{c}))\right]=1
\end{align*}
by Jensen's inequality again, which implies $L_{un}(f)\ge \tau$.
In general, without any further assumptions on $f$, $L_{un}(f)$ can be far from $\tau$, rendering the bound in Theorem \ref{thm:unsup_upper_bound} useless.
However, as we will show, the magnitude of $L_{un}^=(f)$ can be controlled by the intraclass deviation of $f$. Let $\Sigma(f,c)$ the covariance matrix of $f(x)$ when $x\sim \gD_c$.
We define a notion of intraclass deviation as follows:
\begin{align}\label{eq:deviation}
s(f)\coloneqq\mathop{\E}\limits_{c\sim \nu} \Big[\sqrt{\|\Sigma(f,c)\|_2} \mathop{\E}\limits_{x\sim \gD_c}\|f(x)\|\Big]
\end{align}


\begin{lemma}\label{lemma:deviation}
For all $f\in \gF$,
\begin{equation*}
 L_{un}^=(f)-1\leq c's(f)
\end{equation*}
where $c'$ is a positive constant. 
\end{lemma}
We prove Lemma~\ref{lemma:deviation} in Appendix~\ref{lemma:general_deviation}.
Theorem \ref{thm:unsup_upper_bound} combined with Equation~(\ref{eq:decomposition}) and Lemma~\ref{lemma:deviation} gives the following result.

\begin{theorem}\label{thm:binary_theorem}
With probability at least $1-\delta$, $\forall f\in \gF$
\begin{equation*} \label{main}
L_{sup}(\widehat{f})\leq  L_{sup}^{\mu}(\widehat{f}) \leq  L^{\neq}_{un}(f)+\beta \ s(f) +\eta \ Gen_M
\end{equation*}
where $\beta= c'\frac{\tau }{1-\tau}$, $\eta=\frac{1 }{1-\tau}$ and $c'$ is a constant.
\end{theorem}
The above bound highlights two sufficient properties of the function class for unsupervised learning to work: when the function class $\gF$ is rich enough to contain {\em some $f$}  with low $\beta s(f)$ as well as low $L_{un}^{\neq}(f)$ then $\widehat f$, the empirical minimizer of the unsupervised loss -- learned using sufficiently large number of samples -- will have good performance on supervised tasks (low $L_{sup}(\widehat f))$.
\section{Towards Competitive Guarantees}\label{sec:competitive_guarantees}
We provide intuition and counter-examples for why contrastive learning does not always pick the best supervised representation $f \in \gF$ and show how our bound captures these.
Under additional assumptions, we show a competitive bound where classification is done using the mean classifier.
\subsection{Limitations of contrastive learning}\label{subsec:counter}
The bound provided in Theorem~\ref{thm:binary_theorem} might not appear as the most natural guarantee for the algorithm.
Ideally one would like to show a bound like the following: for all $f\in\gF$,
\vspace*{-0.05in}
\begin{align}
\text{(\em Ideal 1):}\quad L_{sup}(\widehat{f})\leq \alpha L_{sup}(f)+\eta\ Gen_{M}\label{eq:ideal1}
\end{align}
for constants $\alpha, \eta$ and generalization error $Gen_{M}$.
This guarantees that $\widehat{f}$ is competitive against the {\em best $f$} on the average binary classification task. However, the bound we prove has the following form: for all $f\in \gF$,
\begin{align*}
L^\mu_{sup}(\widehat{f}) \leq \alpha L_{un}^{\neq}(f) + \beta s(f) + \eta\ Gen_M
\end{align*}
To show that this discrepancy is not an artifact of our analysis but rather stems from limitations of the algorithm, we present two examples in Figure \ref{fig:counter}. Our bound appropriately captures these two issues individually owing to the large values of $L^{\neq}(f)$ or $s(f)$ in each case, for the optimal $f$.

In Figure \ref{fig:counter1}, we see that there is a direction on which $f_1$ can be projected to perfectly separate the classes.
Since the algorithm takes inner products between the representations, it inevitably considers the spurious components along the orthogonal directions.
This issue manifests in our bound as the term $L^{\neq}_{un}(f_1)$ being high even when $s(f_1)=0$.
Hence, contrastive learning will not always work when the only guarantee we have is that $\gF$ can make $L_{sup}$ small.

This should not be too surprising, since we show a relatively strong guarantee -- a bound on $L^{\mu}_{sup}$ for the {\em mean classifier} of $\widehat f$.
This suggests a natural stronger assumption that $\gF$ can make $L_{sup}^\mu$ small (which is observed experimentally in Section~\ref{sec:experiment} for function classes of interest) and raises the question of showing a bound that looks like the following: for all $f\in \gF$,
\begin{equation}\label{dream2}
\text{(\em Ideal 2):}\quad L_{sup}^\mu(\widehat{f})\leq \alpha L_{sup}^\mu(f)+ \eta Gen_{M}
\end{equation}
without accounting for any intraclass deviation -- recall that $s(f)$ captures a notion of this deviation in our bound.
However this is not true: high intraclass deviation may not imply high $L_{sup}^\mu(f)$, but can make $L_{un}^=(f)$ (and thus $L_{un}(f)$) high, resulting in the failure of the algorithm.
Consequently, the term $s(f)$ also increases while $L_{un}^{\neq}$ does not necessarily have to.
This issue, apparent in Figure \ref{fig:counter2}, shows that a guarantee like (\ref{dream2}) cannot be shown without further assumptions.  
\begin{figure}[!t]
  \begin{subfigure}[b]{0.30\columnwidth}\label{fig:counter_figure1}
    \centering
    \includegraphics[width=\linewidth]{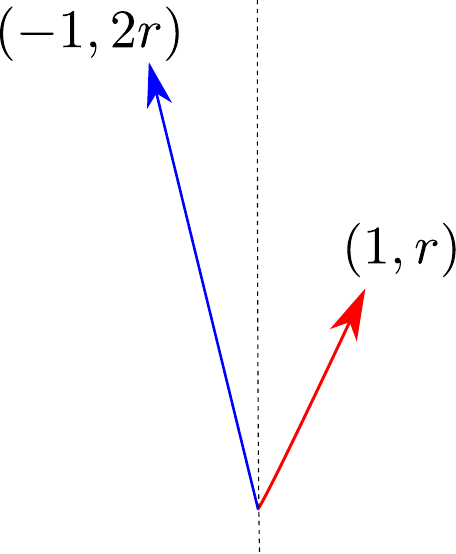}
    \caption{Mean is bad}
    \label{fig:counter1}
  \end{subfigure}
  \hfill
  \begin{subfigure}[b]{0.48\columnwidth}\label{fig:counter_figure2}
    \centering
    \includegraphics[scale=0.25]{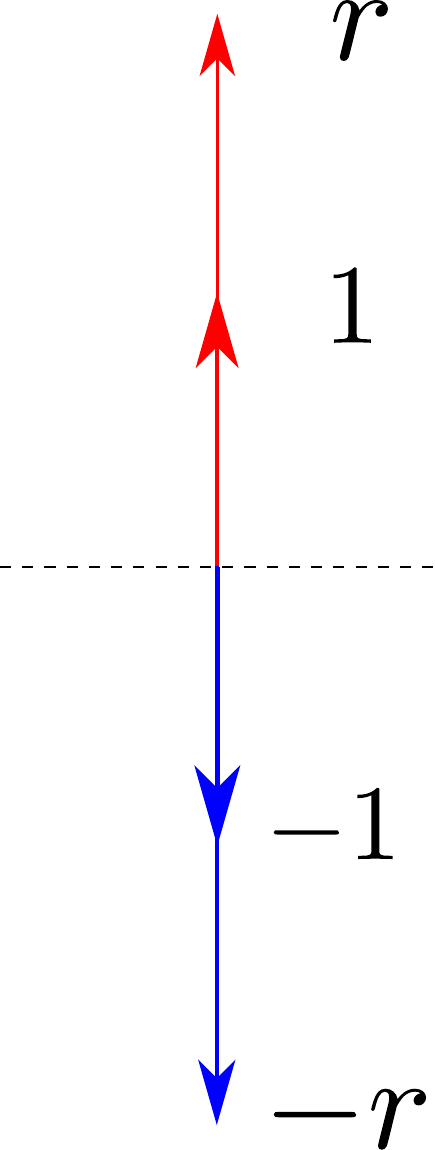}
    \caption{High intraclass variance}
    \label{fig:counter2}
  \end{subfigure}
  \caption{{\em In both examples we have uniform distribution over classes $\gC=\{c_1,c_2\}$, blue and red points are in $c_1$ and $c_2$ respectively and $\gD_{c_i}$ is uniform over the points of $c_i$. In the first figure we have one point per class, while in the second we have two points per class. Let $\gF=\{f_0,f_1\}$ where $f_0$ maps all points to $(0,0)$ and $f_1$ is defined in the figure. In both cases, using the hinge loss, $L_{sup}(f_1)=0$, $L_{sup}(f_0) = 1$ and in the second case $L_{sup}^\mu(f_1)=0$. However, in both examples the algorithm will pick $f_0$ since $L_{un}(f_0) = 1$ but $L_{un}(f_1) = \Omega(r^2)$.}}\label{fig:counter}
\end{figure}

\subsection{Competitive Bound via Intraclass Concentration}\label{subsec:subgaussian}
We saw that $L_{sup}^\mu(f)$ being small does not imply low $L_{sup}^\mu(\widehat{f})$, if $f$ is not concentrated within the classes. In this section we show that when there is an $f$ that has intraclass concentration in a strong sense (sub-Gaussianity) and can separate classes with high margin (on average) with the mean classifier, then $L_{sup}^\mu(\widehat f)$ will be low. 

Let $\ell_\gamma(x)=(1-\frac{x}{\gamma})_+$ be the hinge loss with margin $\gamma$ and $L_{\gamma,sup}^\mu(f)$ be $L_{sup}^\mu(f)$ with the loss function $\ell_{\gamma}$.
\begin{lemma}\label{lemma:subgaussian}
For $f \in \gF$, if the random variable $f(X)$, where $X\sim D_c$, is $\sigma^2$-sub-Gaussian in every direction for every class $c$ and has maximum norm $R = max_{x\in \gX} \|f(x)\|$, then for all $\epsilon > 0$,
\begin{equation*}
\begin{split}
L_{un}^{\neq}(f)\leq \gamma L_{\gamma,sup}^\mu(f)+\epsilon
\end{split}
\end{equation*}
where $\gamma= 1+c'R\sigma\sqrt{\log \frac{R}{\epsilon}}$ and $c'$ is some constant.
\end{lemma}
The proof of Lemma \ref{lemma:subgaussian} is provided in the Appendix \ref{appdx:proof_subgaussian}.
Using Lemma~\ref{lemma:subgaussian} and Theorem~\ref{thm:binary_theorem}, we get the following:
\begin{corollary}\label{corr:subgaussian}
For all $\epsilon>0$, with probability at least $1-\delta$, for all $f\in \gF$,
\begin{equation*}
L_{sup}^{\mu}(\widehat{f}) \leq \gamma(f) L_{\gamma(f), sup}^\mu(f) + \beta s(f) + \eta Gen_M+ \epsilon
\end{equation*}
where $\gamma(f)$ is as defined in Lemma \ref{lemma:subgaussian}, $\beta= c'\frac{\tau }{1-\tau}$, $\eta=~\frac{\tau }{1-\tau}$ and $c'$ is a constant.
\end{corollary}


%
%
%
%
%
%

\section{Multiple Negative Samples and Block Similarity}\label{sec:extension}
In this section we explore two extensions to our analysis.
First, in Section~\ref{subsec:k-way_guarantees}, inspired by empirical works like \citet{Logeswaran:18} that often use more than one negative sample for every similar pair, we show provable guarantees for this case by careful handling of class collision.
Additionally, in Section~\ref{subsec:k-way_effect} we show simple examples where increasing negative samples beyond a certain threshold can hurt contrastive learning.
Second, in Section~\ref{subsec:CURL}, we explore a modified algorithm that leverages access to {\em blocks} of similar data, rather than just pairs and show that it has stronger guarantees as well as performs better in practice.

\subsection{Guarantees for $k$ Negative Samples}\label{subsec:k-way_guarantees}
Here the algorithm utilizes $k$ negative samples $x_1^-,...,x_{k}^-$ drawn i.i.d. from $\gD_{neg}$ for every positive sample pair $x,x^+$ drawn from $\gD_{sim}$ and minimizes (\ref{def:emp_unsuploss}).
As in Section~\ref{sec:power_of_framework}, we prove a bound for $\widehat{f}$ of the following form:
\begin{theorem}\label{thm:multiclass_theorem}(Informal version) For all $f\in\gF$
\begin{equation*}
\gL_{sup}(\widehat f)\le\gL_{sup}^\mu(\widehat{f})\leq \alpha L_{un}^{\neq}(f)+\beta s(f) + \eta\ Gen_{M}
\end{equation*}
\end{theorem}
where $L_{un}^{\neq}(f)$ and $Gen_M$ are extensions of the corresponding terms from Section~\ref{sec:power_of_framework} and $s(f)$ remains unchanged.
The formal statement of the theorem and its proof appears in Appendix~\ref{appdx:multiclass}.
The key differences from Theorem~\ref{thm:binary_theorem} are $\beta$ and the distribution of tasks in $\gL_{sup}$ that we describe below.
The {\em coefficient $\beta$} of $s(f)$ increases with $k$, e.g. when $\rho$ is uniform and $k\ll|\gC|$, $\beta\approx \frac{k}{|\gC|}$.

The {\em average supervised loss} that we bound is
\begin{equation*}
\gL_{sup}(\widehat{f}):=\mathop{\E}\limits_{\gT\sim \gD}\left[L_{sup}(\gT,\widehat{f})\right]
\end{equation*}
where $\gD$ is a distribution over tasks, defined as follows: sample $k+1$ classes $c^+,c_1^-,\dots,c_k^- \sim \rho^{k+1}$, conditioned on the event that $c^+$ does not also appear as a negative sample. Then, set $\gT$ to be the set of distinct classes in $\{c^+,c_1^-,\dots,c_k^-\}$.
$\gL_{sup}^\mu(\widehat{f})$ is defined by using $L^{\mu}_{sup}(\gT,\widehat{f})$.
\begin{remark}
Bounding $\gL_{sup}(\widehat{f})$ directly gives a bound for average $(k+1)$-wise classification loss $L_{sup}(\widehat{f})$ from Definition~\ref{def:sup_loss}, since $L_{sup}(\widehat{f})\leq \gL_{sup}(\widehat{f})/p$, where $p$ is the probability that the $k+1$ sampled classes are distinct.
For $k\ll|\gC|$ and $\rho$ $\approx$ uniform, these metrics are almost equal.
\end{remark}
We also extend our competitive bound from Section \ref{subsec:subgaussian} for the above $\widehat{f}$ in Appendix \ref{appdx:subgaussian_multiclass}.

\subsection{Effect of Excessive Negative Sampling}\label{subsec:k-way_effect}
The standard belief is that increasing the number of negative samples always helps, at the cost of increased computational costs.
In fact for Noise Contrastive Estimation (NCE) \cite{Gutmann:10}, which is invoked to explain the success of negative sampling, increasing negative samples has shown to provably improve the asymptotic variance of the learned parameters.
However, we find that such a phenomenon does not always hold for contrastive learning -- larger $k$ can hurt performance for the same inherent reasons highlighted in Section \ref{subsec:counter}, as we illustrate next.

When $\rho$ is close to uniform and the number of negative samples is $k = \Omega(|\gC|)$, frequent class collisions can prevent the unsupervised algorithm from learning the representation $f \in \gF$ that is optimal for the supervised problem.
In this case, owing to the contribution of $s(f)$ being high, a large number of negative samples could hurt. 
This problem, in fact, can arise even when the number of negative samples is much smaller than the number of classes.
For instance, if the best representation function $f \in \gF$ groups classes into $t$ ``clusters",\footnote{This can happen when $\gF$ is not rich enough.} such that $f$ cannot contrast well between classes from the same cluster, then $L^{\neq}_{un}$ will contribute to the unsupervised loss being high even when $k=\Omega(t)$.
We illustrate, by examples, how these issues can lead to picking suboptimal $\widehat f$ in Appendix~\ref{appdx:multiclass_examples}.
Experimental results in Figures \ref{fig:NvsK_vision} and \ref{fig:NvsK_nlp} also suggest that larger negative samples hurt performance beyond a threshold, confirming our suspicions.


\subsection{Blocks of Similar Points}\label{subsec:CURL}
Often a dataset consists of {\em blocks} of similar data instead of just pairs: a block consists of $x_0, x_1, \dots x_b$ that are i.i.d. draws from a class distribution $D_{c}$ for a class $c \sim \rho $.
In text, for instance, paragraphs can be thought of as a {\em block} of sentences sampled from the same latent class. How can an algorithm leverage this additional structure? 

We propose an algorithm that uses two blocks: one for positive samples $x, x^+_1, \dots, x^+_b$ that are i.i.d. samples from $c^+\sim\rho$ and another one of negative samples $x^-_{1}, \dots x^-_{b}$ that are i.i.d. samples from $c^-\sim\rho$.
Our proposed algorithm then minimizes the following loss:
\begin{equation}\label{eq:block_loss}
\begin{split}
&L_{un}^{block}(f) \coloneqq \\
&\ \ \mathop{\E}\left[ \ell\left(f(x)^T\left(\frac{\sum_i f(x^+_i)}{b} - \frac{\sum_{i} f(x^-_i)}{b}\right) \right) \right]
\end{split}
\end{equation}
To understand why this loss function make sense, recall that the connection between $L^\mu_{sup}$ and $L_{un}$ was made in Lemma \ref{lemma:jensen} by applying Jensen's inequality.
Thus, the algorithm that uses the average of the positive and negative samples in blocks as a proxy for the classifier instead of just one point each should have a strictly better bound owing to the Jensen's inequality getting tighter.
%
We formalize this intuition below. Let $\tau$ be as defined in Section \ref{sec:power_of_framework}. 
\begin{prop}\label{prop:blocks} $\forall f \in \gF$
\begin{align*}
L_{sup}(f) \leq \frac{1}{1-\tau}\left(L_{un}^{block}(f) - \tau\right)\leq \frac{1}{1-\tau}\left(L_{un}(f) - \tau\right)
\end{align*}
\end{prop}
This bound tells us that $L_{un}^{block}$ is a better surrogate for $L_{sup}$, making it a more attractive choice than $L_{un}$ when larger blocks are available.\footnote{Rigorous comparison of the generalization errors is left for future work.}.
The algorithm can be extended, analogously to Equation~(\ref{eq:QT}), to handle more than one negative block.
Experimentally we find that minimizing $L_{un}^{block}$ instead of $L_{un}$ can lead to better performance and our results are summarized in Section \ref{exp_CURL}.
We defer the proof of Proposition \ref{prop:blocks} to Appendix \ref{appdx:proof_blocks}.


\section{Related Work}\label{sec:related}

The contrastive learning framework is inspired by several empirical works, some of which were mentioned in the introduction.
The use of co-occurring words as semantically similar points and negative sampling for learning word embeddings was introduced in \citet{Mikolov:13}.
Subsequently, similar ideas have been used by \citet{Logeswaran:18} and \citet{Pagliardini:18} for sentences representations and by \citet{Wang:15} for images.
Notably the sentence representations learned by the {\em quick thoughts (QT)} method in \citet{Logeswaran:18} that we analyze has state-of-the-art results on many text classification tasks.
Previous attempts have been made to explain negative sampling \cite{Dyer:14} using the idea of Noise Contrastive Estimation (NCE) \cite{Gutmann:10} which relies on the assumption that the data distribution belongs to some known parametric family.
This assumption enables them to consider a broader class of distributions for negative sampling.
The mean classifier that appears in our guarantees is of significance in meta-learning and is a core component of ProtoNets \cite{Snell:17}.

Our data model for similarity is reminiscent of the one in {\em co-training} \cite{Blum:98}. They assume access to pairs of \textquotedblleft views\textquotedblright~with the same label that are conditionally independent given the label.
Our unlabeled data model can be seen as a special case of theirs, where the two views have the same conditional distributions.
However, they additionally assume access to some labeled data (semi-supervised), while we learn representations using only unlabeled data, which can be subsequently used for classification when labeled data is presented.
{\em Two-stage kernel learning}~\cite{Cortes:10,Kumar:12} is similar in this sense: in the first stage, a positive linear combination of some base kernels is learned and is then used for classification in the second stage; they assume access to labels in both stages.
{\em Similarity/metric learning} \cite{Bellet:12, Bellet:13} learns a linear feature map that gives low distance to similar points and high to dissimilar.
While they identify dissimilar pairs using labels, due to lack of labels we resort to negative sampling and pay the price of class collision.
While these works analyze linear function classes, we can handle arbitrarily powerful representations.
Learning of representations that are broadly useful on a distribution of tasks is done in {\em multitask learning}, specifically in the {\em learning-to-learn model} \cite{Maurer:16} but using labeled data.

Recently \citet{Hazan:16} proposed ``assumption-free" methods for representation learning via MDL/compression arguments,
but do not obtain any guarantees comparable to ours on downstream classification tasks. 
As noted by \citet{Arora:17b}, this compression approach has to preserve {\em all} input information (e.g. preserve every pixel of the image) which seems suboptimal.

\begin{table}[t!]\label{table:sup_mean}
\begin{center}
\begin{sc}
	\caption{
		Performance of supervised and unsupervised representations on  average $k$-wise classification tasks ({\sc avg-$k$}) and for comparison, on full multiclass ({\sc top-r}) which is not covered by our theory. Classifier can have a trained output layer ({\sc Tr}), or the mean classifier ({\sc $\mu$}) of Definition~\ref{def:avg_sup_mean}, with $\mu$-5  indicating the mean was computed using only 5 labeled examples. 
	}\label{table:sup_vs_unsup}
		\vskip 0.15in
		\fontsize{8.5}{10}\selectfont
		\begin{tabular}{@{}l@{\hspace{2mm}}|c|@{\hspace{1.5mm}}c@{\hspace{1.5mm}}c@{\hspace{1.5mm}}c@{\hspace{1.5mm}}|@{\hspace{1.5mm}}c@{\hspace{1.5mm}}c@{\hspace{1.5mm}}c@{}}
		    && \multicolumn{3}{c}{Supervised} & \multicolumn{3}{c}{Unsupervised}\\
		    && Tr & $\mu$ & $\mu$-5 & Tr & $\mu$ & $\mu$-5\\
			\toprule
			\multirow{4}*{Wiki-3029}
			& avg-2 & 97.8  & 97.7  & 97.0  & 97.3  & 97.7 & 96.9\\
			& avg-10 & 89.1 & 87.2 & 83.1 & 88.4 & 87.4 & 83.5\\
			\cmidrule{2-8}
			& top-10 & 67.4 & 59.0 & 48.2 & 64.7 & 59.0 & 45.8\\
			& top-1 & 43.2 & 33.2 & 21.7 & 38.7 & 30.4 & 17.0\\
						
			\midrule
			
			\multirow{4}*{CIFAR-100}
			& avg-2 & 97.2 & 95.9 & 95.8 & 93.2 & 92.0 & 90.6\\
			& avg-5 & 92.7 & 89.8 & 89.4 & 80.9 & 79.4 & 75.7\\
			\cmidrule{2-8}
			& top-5 & 88.9 & 83.5 & 82.5 & 70.4 & 65.6 & 59.0\\
			& top-1 & 72.1 & 69.9 & 67.3 & 36.9 & 31.8 & 25.0\\
			\bottomrule
		\end{tabular}
\end{sc}
\end{center}
\end{table}

\section{Experimental Results}\label{sec:experiment}
We report experiments in text and vision domains supporting our theory.
Since contrastive learning has already shown to obtain state-of-the-art results on text classification by {\em quick thoughts} (QT) in \citet{Logeswaran:18}, most of our experiments are conducted to corroborate our theoretical analysis.
We also show that our extension to similarity blocks in Section~\ref{subsec:CURL} can improve QT on a real-world task.

{\bf Datasets:} 
Two datasets were used in the controlled experiments. 
(1) The CIFAR-100 dataset~\cite{Krizhevsky:09} consisting of 32x32 images categorized into 100 classes with a 50000/10000 train/test split.
(2) Lacking an appropriate NLP dataset with large number of classes, we create the Wiki-3029 dataset, consisting of 3029 Wikipedia articles as the classes and 200 sentences from each article as samples.
The train/dev/test split is 70\%/10\%/20\%. 
To test our method on a more standard task,  we also use the unsupervised part of the IMDb review corpus \cite{Maas:11}, which consists of 560K sentences from 50K movie reviews. Representations trained using this corpus are evaluated on the supervised IMDb binary classification task, consisting of training and testing set with 25K reviews each.

\begin{figure*}[t!]
\begin{subfigure}[b]{0.32\linewidth}
    \includegraphics[width=\linewidth]{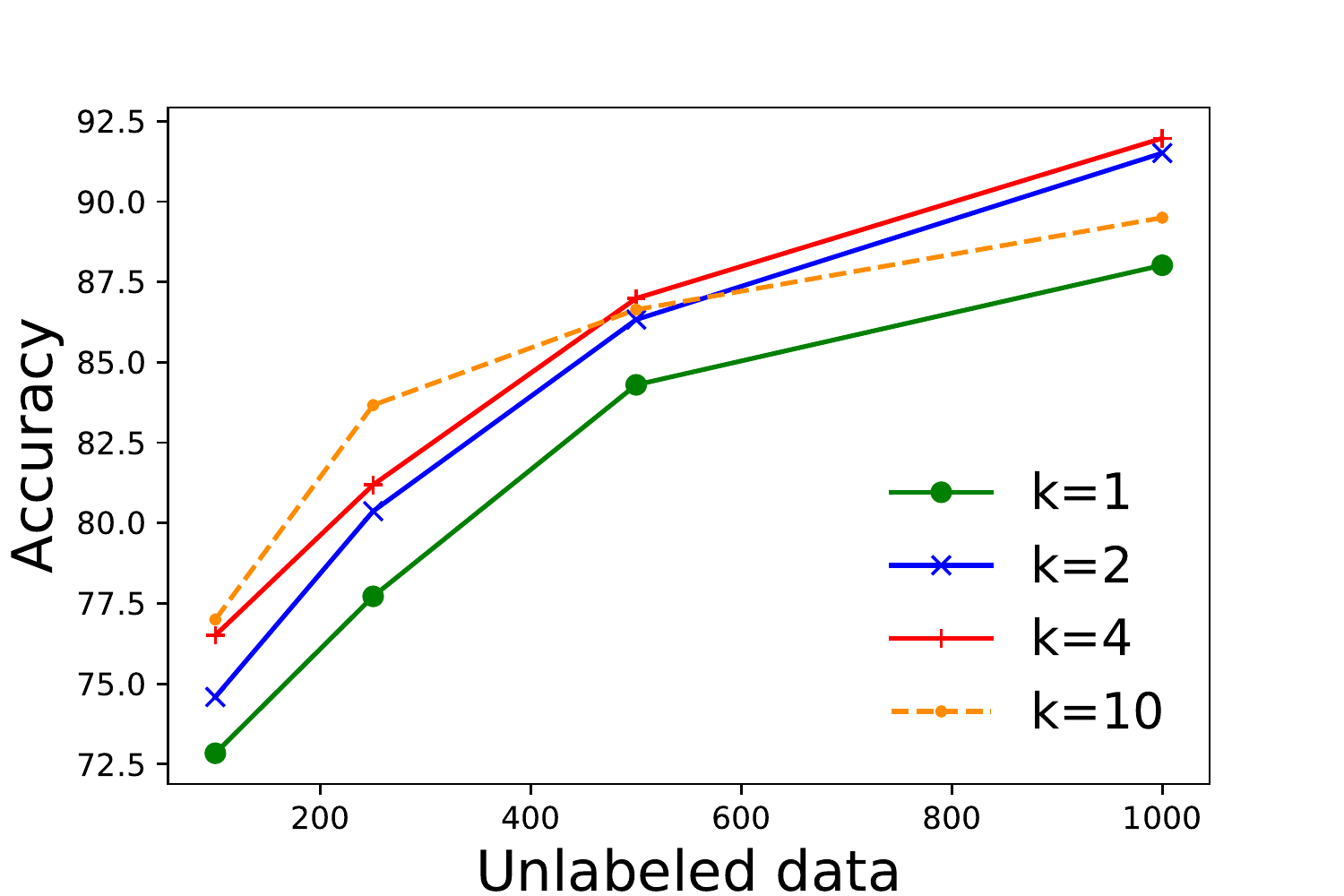}
    \caption{CIFAR-100}
    \label{fig:NvsK_vision}
  \end{subfigure}
  \hfill 
  \begin{subfigure}[b]{0.32\linewidth}
    \includegraphics[width=\linewidth]{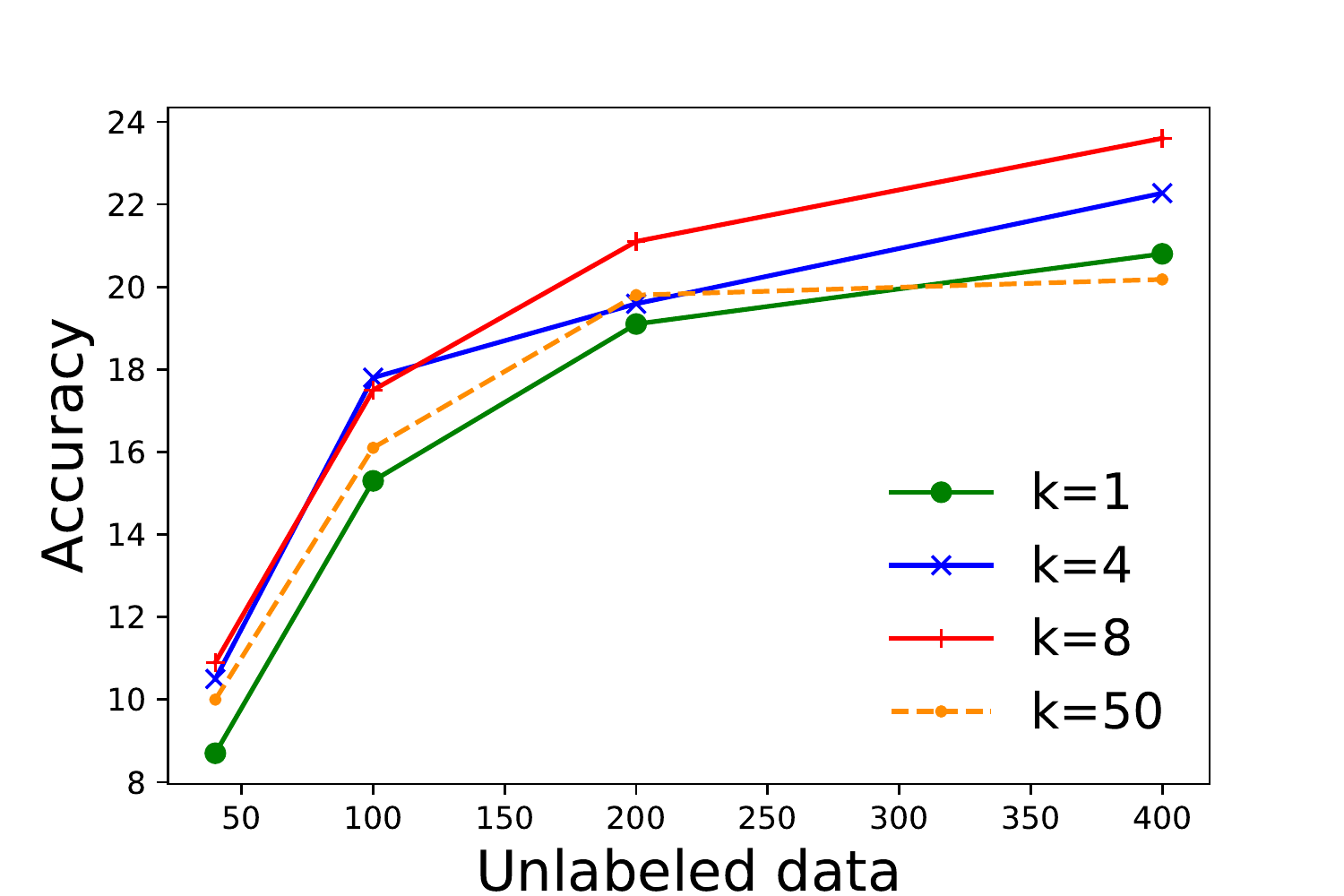}
    \caption{Wiki-3029}
    \label{fig:NvsK_nlp}
  \end{subfigure}
  \hfill 
  \begin{subfigure}[b]{0.32\linewidth}
    \includegraphics[width=\linewidth]{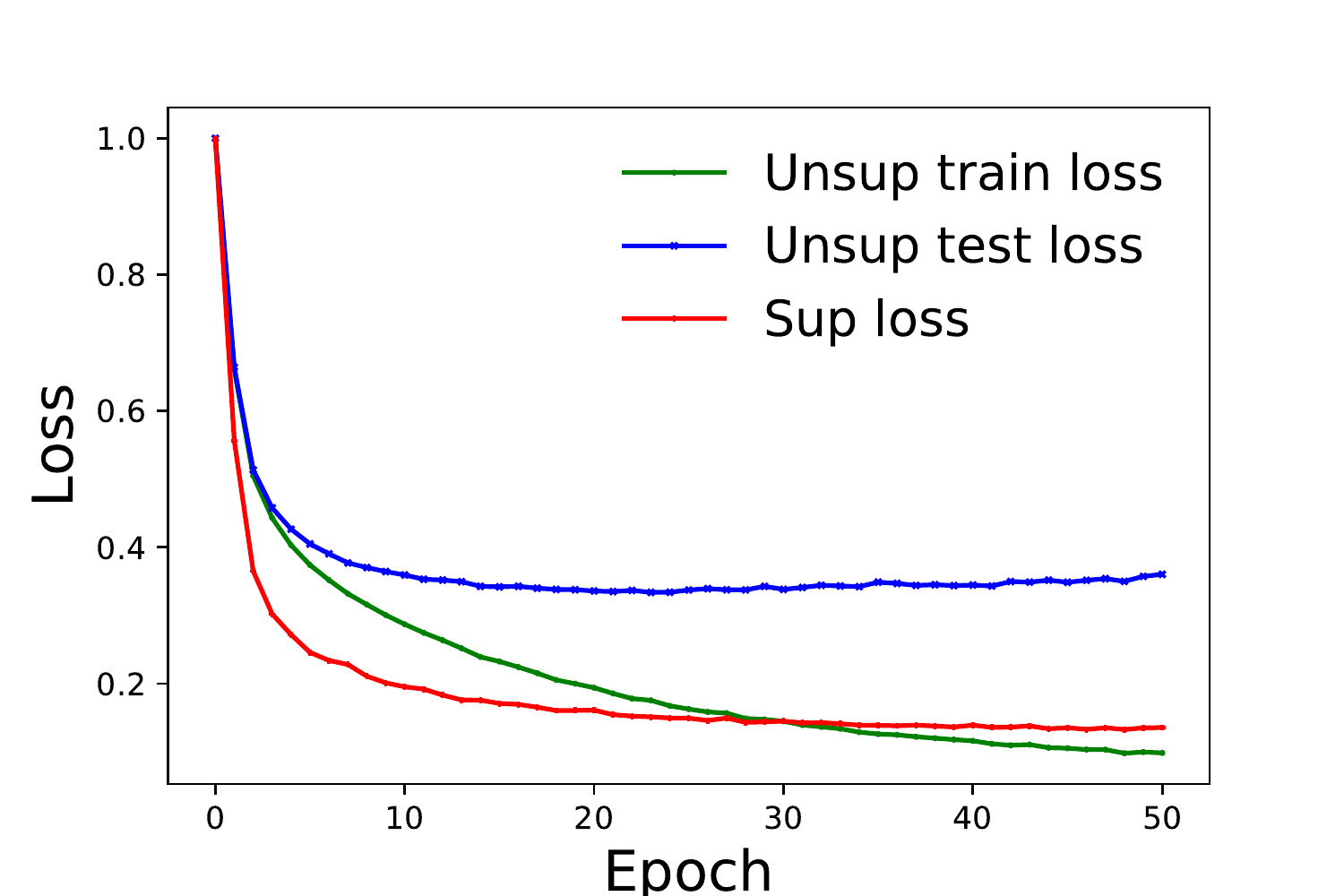}
    \caption{Wiki-3029}
    \label{fig:unsup_loss}
  \end{subfigure}
      \caption{Effect of amount of unlabeled data and \# of negative samples on unsupervised representations, measured on binary classification for CIFAR100 in (a) and on top-1 performance on Wiki-3029 in Fig (b) (top-1 performance is used because avg binary was same for all $k$). Fig. (c) shows the dynamics of train/test loss;  supervised loss roughly tracks unsupervised test loss, as suggested by Theorem~\ref{thm:unsup_upper_bound}}\label{fig:NvsK}
\end{figure*}

\begin{table}[t!]
\begin{center}
\begin{sc}
	\centering
	\caption{Effect of larger block size on representations. For {\sc CIFAR-100} and {\sc Wiki-3029} we measure the average binary classification accuracy. {\sc IMDb} representations are tested on {\sc IMDb} supervised task. CURL is our large block size contrastive method, QT is the algorithm from \cite{Logeswaran:18}. For larger block sizes, QT uses all pairs within a block as similar pairs. We use the same GRU architecture for both CURL and QT for a fair comparison.}\label{table:block_size}
		\vskip 0.15in
		\fontsize{9}{10}\selectfont
		\begin{tabular}{@{}l|c|ccc@{}}
			\toprule
		    	Dataset & Method & $b=2$ & $b=5$ & $b=10$ \\
			\midrule
			\multirow{1}*{CIFAR-100}
			 & CURL & 88.1 & 89.6 & 89.7\\
			\multirow{1}*{Wiki-3029}
			 & CURL & 96.6 & 97.5 & 97.7\\
			 \midrule
			\multirow{2}*{IMDb}
			& CURL & 89.2 & 89.6 & 89.7 \\
			& QT & 86.5 & 87.7 & 86.7\\
			\bottomrule
		\end{tabular}
\end{sc}
\end{center}
\end{table}

\subsection{Controlled Experiments}\label{sec:controlled_exp}
To simulate the data generation process described in Section~\ref{sec:framework}, we generate similar pairs (blocks) of data points by sampling from the same class.
Dissimilar pairs (negative samples) are selected randomly.
Contrastive learning was done using our objectives (\ref{eq:QT}), and compared to performance of standard supervised training, with both using the {\em same architecture} for representation $f$.
For CIFAR-100 we use VGG-16 \cite{Simonyan:14} with an additional 512x100 linear layer added at the end to make the final representations 100 dimensional, while for Wiki-3029 we use a Gated Recurrent Network (GRU) \cite{Chung:15} with output dimension 300 and fix the word embedding layer with pretrained GloVe embeddings \cite{Pennington:14}.
The unsupervised model for CIFAR-100 is trained with 500 blocks of size 2 with 4 negative samples, and for Wiki-3029 we use 20 blocks of size 10 with 8 negative samples.
We test (1) learned representations on average tasks by using the mean classifier and compare to representations trained using labeled data; (2) the effect of various parameters like amount of unlabeled data ($N$)\footnote{If we used $M$ similar blocks of size $b$ and $k$ negative blocks for each similar block, $N=Mb(k+1)$. In practice, however, we reuse the blocks for negative sampling and lose the factor of $k+1$.}, number of negative samples ($k$) and block size ($b$) on representation quality; (3)
whether the supervised loss tracks the unsupervised loss as suggested by Theorem~\ref{thm:unsup_upper_bound}; (4) performance of the mean classifier of the supervised model.

{\bf Results:}
These appear in Table~\ref{table:sup_vs_unsup}. For Wiki-3029 the unsupervised performance is very close to the supervised performance in all respects, while for CIFAR-100 the {\em avg-$k$} performance is respectable, rising to good for binary classification.
One surprise is that the mean classifier, central to our analysis of unsupervised learning, performs well also with representations learned by supervised training on CIFAR-100.
Even the mean computed by just $5$ labeled samples performs well, getting within $2\%$ accuracy of the $500$ sample mean classifier on CIFAR-100.
This suggests that representations learnt by standard supervised deep learning are actually quite concentrated.
We also notice that the supervised representations have fairly low unsupervised training loss (as low as 0.4), even though the optimization is minimizing a different objective. 

To measure the sample complexity benefit provided by contrastive learning, we train the supervised model on just $10\%$ fraction of the dataset and compare it with an unsupervised model trained on unlabeled data whose mean classifiers are computed using the same amount of labeled data.
We find that the unsupervised model beats the supervised model by almost $4\%$ on the 100-way task and by $5\%$ on the average binary task when only 50 labeled samples are used.

Figure \ref{fig:NvsK} highlights the positive effect of increasing number of negative samples as well as amount of data used by unsupervised algorithm. 
In both cases, using a lot of negative examples stops helping after a point, confirming our suspicions in Section~\ref{subsec:k-way_effect}.
We also demonstrate how the supervised loss tracks unsupervised test loss in Figure~\ref{fig:unsup_loss}.

\subsection{Effect of Block Size}\label{exp_CURL}

As suggested in Section~\ref{subsec:CURL}, a natural extension to the model would be access to blocks of similar points.
We refer to our method of minimizing the loss in (\ref{eq:block_loss}) as {\em CURL} for {\em Contrastive Unsupervised Representation Learning} and perform experiments on CIFAR-100, Wiki-3029, and IMDb.
In Table~\ref{table:block_size} we see that for CIFAR-100 and Wiki-3029, increasing  block size yields an improvement in classification accuracy.
For IMDb, as is evident in Table~\ref{table:block_size}, using larger blocks provides a clear benefit and the method does better than QT, which has state-of-the-art performance on many tasks.
A thorough evaluation of CURL and its variants on other unlabeled datasets is left for future work.



\section{Conclusion}\label{sec:conclusion}

Contrastive learning methods have been empirically successful at learning useful feature representations.
We provide a new conceptual framework for thinking about this form of learning, which also allows us to formally treat issues such as guarantees on the quality of the learned representations.
The framework gives fresh insights into what guarantees are possible and impossible, and shapes the search for new assumptions to add to the framework that allow tighter guarantees.
The framework currently ignores issues of efficient minimization of various loss functions, and instead studies the interrelationships of their minimizers as well as sample complexity requirements for training to generalize, while clarifying what generalization means in this setting.
Our approach should be viewed as a first cut; possible extensions include allowing tree structure -- more generally metric structure -- among the latent classes. Connections to meta-learning and transfer learning may arise.

We use experiments primarily to illustrate and support the new framework. But one experiment on sentence embeddings already illustrates how fresh insights derived from our framework can lead to improvements upon state-of-the-art models in this active area. We hope that further progress will follow, and that our theoretical insights will begin to influence practice, including design of new heuristics to identify semantically similar/dissimilar pairs.

\section{Acknowledgements}
This work is supported by NSF, ONR, the Simons Foundation, the Schmidt Foundation, Mozilla Research, Amazon Research, DARPA, and SRC.
We thank Rong Ge, Elad Hazan, Sham Kakade, Karthik Narasimhan, Karan Singh and Yi Zhang for helpful discussions and suggestions.

\bibliography{references}

\bibliographystyle{icml2019}

\appendix

\onecolumn
\section{Deferred Proofs}
\subsection{Class Collision Lemma}
We prove a general Lemma, from which Lemma \ref{lemma:deviation} can be derived directly.


\begin{lemma}\label{lemma:general_deviation}
Let $c\in \gC$ and $\ell:\mathbb{R}^t\rightarrow \mathbb{R}$ be either the $t$-way hinge loss or $t$-way logistic loss, as defined in Section \ref{sec:framework}. 
Let $x,x^+,x_1^-,...,x_t^-$ be iid draws from $\gD_c$. For all $f\in \gF$, let 
$$L_{un,c}^=(f)=\mathop{\E}\limits_{x,x^+,x_i^-}\left[\ell \left(\left\{f(x)^T\left(f(x^+)-f(x_i^-)\right)\right\}_{i=1}^t\right)\right]$$ 
Then
\begin{equation}\label{eq:fixed_deviation}
 L_{un,c}^=(f)-\ell(\vec{0})\leq c' t \sqrt{\|\Sigma(f,c)\|_2}\mathop{\E}\limits_{x\sim \gD_c}[\|f(x)\|]
\end{equation}
where $c'$ is a positive constant.
\end{lemma}

Lemma \ref{lemma:deviation} is a direct consequence of the above Lemma, by setting $t=1$ (which makes $\ell(0)=1$), taking an expectation over $c \sim \nu$ in \Eqref{eq:fixed_deviation} and noting that $\mathop{\E}_{c \sim \nu}[L_{un,c}^=(f)] = L_{un}^=(f)$. 

\begin{proof}[Proof of Lemma \ref{lemma:general_deviation}]
Fix an $f\in \gF$ and let $z_i=f(x)^T\left(f(x_i^-)-f(x^+)\right)$ and $z=\max_{i\in [t]}{z_i}$. First, we show that  $L_{un,c}^=(f)-\ell(\vec{0})\leq c' \mathbb{E}[|z|]$, for some constant $c'$. Note that $\mathbb{\E}[|z|]=\mathbb{P}[z\geq 0]\mathbb{E}[z| z\geq 0]+\mathbb{P}[z\leq 0]\mathbb{E}[-z | z\leq 0]\geq \mathbb{P}[z\geq 0]\mathbb{E}[z| z\geq 0]$.

{\bf $t$-way hinge loss}: By definition $\ell(\vv)=\max\{0,1+\max_{i\in[t]}\{-\vv_i\}\}$. Here, $L_{un,c}^=(f) =\mathbb{E}[(1+z)_+] \leq \mathbb{E}[\max\{1+z,1\}]=1+\mathbb{P}[z\geq 0]\mathbb{E}[z|z\geq 0]\leq 1+ \mathbb{\E}[|z|]$.

{\bf $t$-way logistic loss}: By definition $\ell(\vv)=\log_2(1+\sum_{i=1}^te^{-\vv_i})$, we have $L_{un,c}^=(f) =\mathbb{E}[\log_2(1+\sum_{i}e^{z_i})]\leq \mathbb{E}[\log_2(1+ te^{z})]\leq \max\{\frac{z}{\log 2}+\log_2(1+t),\log_2(1+t) \}=\frac{\mathbb{P}[z\geq 0]\mathbb{E}[z|z\geq 0]}{\log{2}}+\log_2(1+t)\leq\frac{ \mathbb{\E}[|z|]}{\log{2}}+\log_2(1+t)$. 

Finally, $\mathbb{E}[|z|]\leq \mathbb{E}[\max_{i\in [t]}|z_i|]\leq t \mathbb{E}[|z_1|]$. But, 
\begin{equation*} 
\begin{split}
\mathbb{E}[|z_1|]&=\mathbb{E}_{x,x^+,x_1^-}\left[\ \left|f(x)^T\left(f(x_1^-)-f(x^+)\right)\right|\ \right]\\
&\leq\mathbb{E}_{x}\left[\|f(x)\|\sqrt{\mathbb{E}_{x^+,x_1^-}\left[\ \left(\frac{f(x)^T}{\|f(x)\|}\left(f(x_1^-)-f(x^+)\right)\right)^2\ \right]}\right]\leq\sqrt{2}\ \sqrt{\|\Sigma(f,c)\|_2}\mathop{\E}\limits_{x\sim \gD_c}[\|f(x)\|]
\end{split}
\end{equation*}

\end{proof}

\subsection{Proof of Lemma \ref{lemma:subgaussian}}\label{appdx:proof_subgaussian}
Fix an $f\in \gF$ and suppose that within each class $c$, $f$ is $\sigma^2$-subgaussian in every direction. 
\footnote{A random variable X is called $\sigma ^2$-subgaussian if $\mathop{\E}[e^{\lambda (X-\mathbb{E}[X])}]\leq e^{\lambda^2\sigma^2/2}$, $\forall \lambda \in \mathbb{R}$. A random vector $V \in \R^d$ is $\sigma^2$-subgaussian in every direction, if  $\forall u \in \R^d, ||u|| = 1$, the random variable $\langle u,V \rangle$ is $\sigma^2$-subgaussian.}
Let $\mu_c=\mathop{\E}\limits_{x\sim \gD_c}[f(x)]$. 
This means that for all $c\in\gC$ and unit vectors $v$, for $x \sim D_c$, we have that $v^T(f(x)-\mu_c)$ is $\sigma^2$-subgaussian. Let $\epsilon>0$ and $\gamma=1+2R\sigma\sqrt{2\log{R}+\log{3/\epsilon}}$.  $ \ \footnote{We implicitly assume here that $R\geq1$, but for $R<1$, we just set $\gamma=1+2R\sigma\sqrt{\log{3/\epsilon}}$ and the same argument holds.}$
 Consider fixed $c^+,c^-,x$ and let $f(x)^T(f(x^-)-f(x^+))= \mu+z$, where 
 $$ \mu=f(x)^T(\mu_{c^-}-\mu_{c^+}) \qquad \text{and} \qquad z=f(x)^T\left(f(x^-)-\mu_{c^-}\right)-f(x)^T\left(f(x^+)-\mu_{c^+}\right)$$
For $x^+\sim \gD_c^+$,  $x^-\sim \gD_c^-$ independently, z is the sum of two independent $R^2\sigma^2$-subgaussians ($x$ is fixed), so z is $2R^2\sigma^2$-subgaussian and thus $p=\Pr[z\geq\gamma-1] \leq e^{-\frac{4R^2\sigma^2(2\log{R}+\log{3/\epsilon})}{4R^2\sigma^2}}=\frac{\epsilon}{3R^2}$. So, $\mathop{\E}_{z}[(1+\mu+z)_+]\leq (1-p)(\gamma+\mu)_+ +p(2R^2+1) \leq \gamma(1+ \frac{\mu}{\gamma})_++\epsilon$ (where we used that $\mu+z\leq 2R^2$). By taking expectation over $c^+,c^-\sim \rho^2$, $x\sim \gD_{c^+}$ we have 
\begin{equation}\label{eq:subgaussian_symmetrization}
\begin{split}
L_{un}^{\neq}(f) &\leq \mathop{\E}\limits_{\substack{c^+,c^-\sim\rho^2\\ x\sim \gD_{c^+}}}\left[\gamma\left(1+\frac{f(x)^T(\mu_{c^-}-\mu_{c^+})}{\gamma}\right)_+\bigg{|} c^+\neq c^-\right]+\epsilon \\
&=  \gamma \mathop{\E}\limits_{c^+,c^-\sim\rho^2 }\left[\frac{1}{2} \mathop{\E}\limits_{x\sim \gD_{c^+}}\left[ \left(1+\frac{f(x)^T(\mu_{c^-}-\mu_{c^+})}{\gamma}\right)_+\right]+\frac{1}{2} \mathop{\E}\limits_{ x\sim \gD_{c^-}}\left[ \left(1+\frac{f(x)^T(\mu_{c^+}-\mu_{c^-})}{\gamma}\right)_+\right] \bigg{|} c^+\neq c^- \right]+\epsilon\\
&=\gamma \mathop{\E}\limits_{c^+,c^-\sim\rho^2 }\left[ L_{\gamma,sup}^\mu(\{c^+,c^-\},f)\big{|}c^+\neq c^- \right] +\epsilon
\end{split}
\end{equation}
where $L_{\gamma,sup}^\mu(\{c^+,c^-\},f)$ is  $L_{sup}^\mu(\{c^+,c^-\},f)$ when $\ell_\gamma(x)=(1-x/\gamma)_+$ is the loss function. Observe that in \ref{eq:subgaussian_symmetrization} we used that $\gD_{\gT}$ are uniform for binary $\gT$, which is an assumption we work with in section 4, but we remove it in section 5. The proof finishes by observing that the last line in \ref{eq:subgaussian_symmetrization} is equal to $\gamma L_{\gamma,sup}^\mu(f)+\epsilon$.

 \qed

\subsection{Generalization Bound}\label{appdx:gen_bound}
We first state the following general Lemma in order to bound the generalization error of the function class $\gF$ on the unsupervised loss function $L_{un}(\cdot)$. Lemma \ref{lemma:gen_bound} can be directly derived from it.
\begin{lemma}\label{generalization}
Let $\ell:\mathbb{R}^k\rightarrow\mathbb{R}$ be $\eta$-Lipschitz and bounded by $B$. Then with probability at least $1-\delta$ over the training set $\gS=\{(x_j,x_j^+,x_{j1}^-,\dots,x_{jk}^-)\}_{j=1}^M$, for all $f\in \gF$
\begin{equation}
L_{un}(\hat{f})\leq  L_{un}(f)+O \left(\frac{\eta R \sqrt{k} \mathcal{R}_\gS(\gF)}{M}+  B\sqrt{\frac{\log{\frac{1}{\delta}}}{M}} \right)
\end{equation}
where
\begin{equation}
\gR_\gS(\gF)=\mathop{\E}\limits_{\sigma \sim \{\pm1\}^{(k+2)dM}} \left[ \sup_{f\in \gF} \langle \sigma, f_{|\gS} \rangle \right]
\end{equation}
and  $f_{|\gS}=\left(f_t(x_j),f_t(x_j^+),f_t(x_{j1}^-),\dots, ,f_t(x_{jk}^-)\right)_{\substack{j\in[M] \\ t \in [d]}} $
\end{lemma}

Note that for $k+1$-way classification, for hinge loss we have $\eta=1$ and $B=O(R^2)$, while for logistic loss $\eta=1$ and $B=O(R^2+\log{k})$. Setting $k=1$, we get Lemma \ref{lemma:gen_bound}.
We now prove Lemma \ref{generalization}.

\begin{proof}[Proof of Lemma \ref{generalization}]
First, we use the classical bound for the generalization error in terms of the Rademacher complexity of the function class (see \cite{Mohri:18} Theorem 3.1). For a real function class $G$ whose functions map from a set $Z$ to $[0,1]$ and for any $\delta>0$, if $\gS$ is a training set composed by $M$ iid samples $\{z_j\}_{j=1}^M$, then with probability at least $1-\frac{\delta}{2}$, for all $g\in G$
\begin{equation}
\mathop{\E}[g(z)]\leq \frac{1}{M}\sum_{j=1}^Mg(z_i)+\frac{2\mathcal{R}_\gS(G)}{M}+ 3\sqrt{\frac{\log{\frac{4}{\delta}}}{2M}}
\end{equation}
where $\mathcal{R}_\gS(G)$ is the usual Rademacher complexity.
We apply this bound to our case by setting $Z=\gX^{k+2}$, $\gS$ is our training set and the function class is 
\begin{equation}
G=\left\{g_f(x,x^+,x_{1}^-,...,x_{k}^-)=\frac{1}{B}\ell \left(\{f(x)^T\big(f(x^+)-f(x_i^-)\big)\}_{i=1}^k \right) \Big{|} f\in \gF \right\}
\end{equation}
We will show that for some universal constant c, $\mathcal{R}_{\gS}(G)\leq c \frac{\eta R \sqrt{k}}{B} \mathcal{R}_{\gS}(\gF)$ or equivalently
\begin{equation}\label{gen.bound_goal}
\mathop{\E}\limits_{\sigma \sim \{\pm 1\}^M} \left[  \sup_{f\in \gF} \left \langle \sigma, (g_f)_{|\mathcal{S}}\right \rangle \right] \leq c\frac{\eta R \sqrt{k}}{B} \mathop{\E}\limits_{\sigma \sim \{\pm 1\}^{d(k+2)M}} \left[  \sup_{f\in \gF} \left \langle \sigma, f_{|\mathcal{S}}\right \rangle \right]
\end{equation}
where $(g_f)_{|\mathcal{S}}=\{g_f(x_j,x_j^+,x_{j1}^-,...,x_{jk}^-)\}_{j=1}^M$.
 To do that we will use the following vector-contraction inequality. 
\\
\begin{theorem}\label{Maurer}{[Corollary 4 in \cite{Maurer:16vec}]}
Let $Z$ be any set, and $\mathcal{S}=\{z_j\}_{j=1}^M \in Z^M$. Let $\widetilde{\gF}$ be a class of functions $\tilde{f} :Z \rightarrow \mathbb{R}^n$ and $h : \mathbb{R}^n \rightarrow \mathbb{R}$ be L-Lipschitz. For all $\tilde{f}\in \widetilde{\gF}$, let $g_{\tilde{f}}=h\circ \tilde{f}$. Then
\begin{equation*}
\mathop{\E}\limits_{\sigma \sim \{\pm 1\}^M} \left[  \sup_{\tilde{f}\in \widetilde{\gF}} \left \langle \sigma, (g_{\tilde{f}})_{|\mathcal{S}}\right \rangle \right] \leq \sqrt{2} L \mathop{\E}\limits_{\sigma \sim \{\pm 1\}^{nM}} \left[ \sup_{\tilde{f}\in \widetilde{\gF}} \left \langle \sigma, \tilde{f}_{|\mathcal{S}}\right \rangle \right]
\end{equation*}
where $\tilde{f}_{|\mathcal{S}}=\left(\tilde{f}_t(z_j)\right)_{t\in[n],j\in [M] }$.
\end{theorem}
We apply Theorem \ref{Maurer} to our case by setting $Z=\gX^{k+2}$, $n=d(k+2)$ and 
$$\widetilde{\gF}=\left\{\tilde{f}(x,x^+,x_{j1}^-,...,x_{jk}^-)=(f(x),f(x^+),f(x_{j1}^-),...,f(x_{jk}^-))| f\in \gF \right\}$$ 

We also use $g_{\tilde{f}}=g_f$ where $\tilde{f}$ is derived from $f$ as in the definition of $\widetilde{F}$. Observe that now \ref{Maurer} is exactly in the form of \ref{gen.bound_goal} and we need to show that $L\leq \frac{c}{\sqrt{2}} \frac{\eta R \sqrt{k}}{B}$ for some constant c. But, for $z=(x,x^+,x_{1}^-,...,x_{k}^-)$, we have $g_{\tilde{f}}(z)=\frac{1}{B}\ell(\phi(\tilde{f}(z)))$ where $\phi:\mathbb{R}^{(k+2)d}\rightarrow \mathbb{R}^k$ and $\phi\left((v_t,v_t^+,v_{t1}^-,...,v_{tk}^-)_{t\in[d]}\right)=\left(\sum_t v_{t}(v_t^+-v_{ti}^-)\right)_{i\in[k]}$. Thus, we may use $h=\frac{1}{B} \ell \circ \phi$ to apply Theorem \ref{Maurer}. 

Now, we see that $\phi$ is $\sqrt{6k}R$-Lipschitz when $\sum_t v_t^2, \sum_t (v_t^+)^2,\sum_t (v_{tj}^-)^2 \leq R^2$ by computing its Jacobian. Indeed, for all $i,j\in [k]$ and $t\in[d]$, we have $\frac{\partial\phi_i}{\partial v_t}=v_t^+-v_{ti}^-$, $\frac{\partial\phi_i}{\partial v_t^+}=v_t$ and $\frac{\partial\phi_i}{\partial v_{tj}^-}=-v_t 1\{i=j\}$. From triangle inequaltiy, the Frobenius norm of the Jacobian $J$ of 
$\phi$ is 
$$||J||_F=\sqrt{\sum_{i,t}(v_t^+-v_{ti}^-)^2+2k\sum_{t}v_t^2}\leq \sqrt{4kR^2+2kR^2}=\sqrt{6k}R$$

 Now, taking into account that $||J||_2\leq ||J||_F$, we have that $\phi$ is $\sqrt{6k}R$-Lipschitz on its domain and since $\ell$ is $\eta$-Lipschitz, we have $L\leq \sqrt{6} \frac{\eta R \sqrt{k}}{B}$.

\par Now, we have that with probability at least $1-\frac{\delta}{2}$
\begin{equation}\label{almost_done}
L_{un}(\hat{f})\leq  \widehat{L}_{un}(\hat{f})+O \left(\frac{\eta R \sqrt{k} \mathcal{R}_\gS(\gF)}{M}+  B\sqrt{\frac{\log{\frac{1}{\delta}}}{M}} \right)
\end{equation}
Let $f^*\in \argmin_{f\in \gF}L_{un}(f)$. With probability at least $1-\frac{\delta}{2}$, we have that  $\widehat{L}_{un}(f^*)\leq L_{un}(f^*)+3B\sqrt{\frac{\log{\frac{2}{\delta}}}{2M}}$ (Hoeffding's inequality). Combining this with \Eqref{almost_done}, the fact that $\widehat{L}_{un}(\hat{f})\leq \widehat{L}_{un}(f^*)$ and applying a union bound, finishes the proof.
\end{proof}

\subsection{Proof of Proposition \ref{prop:blocks}}\label{appdx:proof_blocks}
By convexity of $\ell$, 

$$\ell\left(f(x)^T\left(\frac{\sum_i f(x^+_i)}{b} - \frac{\sum_i f(x^-_i)}{b}\right)\right)=\ell\left(\frac{1}{b}\sum_i f(x)^T\left( f(x^+_i) -f(x^-_i)\right)\right) \leq \frac{1}{b}\sum_i \ell\left(f(x)^T\left(f(x^+_i) -  f(x^-_i)\right) \right)  $$
Thus, 
\begin{align*}
L_{un}^{block}(f) = \mathop{\E}_{\substack{x, x^+_i \\ x^-_i }} \left[ \ell\left(f(x)^T\left(\frac{\sum_i f(x^+_i)}{b} - \frac{\sum_i f(x^-_i)}{b}\right)\right) \right] \leq \mathop{\E}_{\substack{x, x^+_i \\ x^-_i }} \left[ \frac{1}{b}\sum_i \ell\left(f(x)^T\left(f(x^+_i) -  f(x^-_i)\right) \right) \right] = L_{un}(f)
\end{align*}
The proof of the lower bound is analogous to that of Lemma \ref{lemma:jensen}.

\qed

\section{Results for k Negative Samples}
\subsection{Formal theorem statement and proof}\label{appdx:multiclass}
We now present Theorem \ref{thm:multiclass_appdx} as the formal statement of Theorem \ref{thm:multiclass_theorem} and prove it. First we define some necessary quantities.

Let $(c^+,c_1^-,\dots,c_k^-)$ be $k+1$ not necessarily distinct classes. We define $Q(c^+,c_1^-,\dots,c_k^-)$ to be the set of distinct classes in this tuple.
We also define $I^+(c_1^-,...,c_k^-)=\{i \in [k] \ |\ c_i^-=c^+\}$ to be the set of indices where $c^+$ reappears in the negative samples.
We will abuse notation and just write $Q$, $I^+$ when the tuple is clear from the context.

To define $L_{un}^{\neq}(f)$ consider the following tweak in the way the latent classes are sampled: sample  $c^+,c_1^-,\dots,c_k^-\sim \rho^{k+1}$ conditioning on $|I^+|< k$ and then remove all $c_i^-$, $i\in I^+$. The datapoints are then sampled as usual: $x,x^+\sim \gD_{c^+}^2$ and $x_i^-\sim \gD_{c_i^-}$, $i\in[k]$, independently.
\begin{align*}
L_{un}^{\neq}(f) \coloneqq \mathop{\E}\limits_{\substack{c^+,c_i^-  \\ x,x^+, x_i^-  }} \left[\ell \left( \left\{f(x)^T\left(f(x^+)  -  f(x_i^-)  \right)\right\}_{i\notin I^+ } \right) \Big{|} |I^+|<k \right]
\end{align*}
which always contrasts points from different classes, since it only considers the negative samples that are not from $c^+$.

The generalization error is \footnote{The $\log{k}$ term can be made $O(1)$ for the hinge loss.}
\begin{equation*}
Gen_{M}=O\left( R\sqrt{k} \frac{\gR_{\gS}(\gF)}{M} + (R^2 +\log{k}) \sqrt{\frac{\log{\frac{1}{\delta}}}{M}}\right) \end{equation*}
were $\gR_{\gS}(\gF)$ is the extension of the definition in Section~\ref{sec:power_of_framework}: $\gR_\gS(\gF)=\mathop{\E}\limits_{\sigma \sim \{\pm1\}^{(k+2)dM}} \left[ \sup_{f\in \gF} \langle \sigma, f_{|\gS} \rangle \right]$, where  $f_{|\gS}=\left(f_t(x_j),f_t(x_j^+),f_t(x_{j1}^-),\dots, ,f_t(x_{jk}^-)\right)_{\substack{j\in[M],t \in [d]}}$.

For $c^+,c_1^-,...,c_k^- \sim \rho^{k+1}$, let $\tau_k=\mathbb{P}[I^+\neq\emptyset]$ and $\tau'=\mathbb{P}[c^+=c_i^-,\forall i]$.
Observe that $\tau_1$, as defined in Section \ref{sec:power_of_framework}, is $\mathbb{P}[c^+=c_1^-]$.
Let $p_{max}(\gT)=\max_c{\gD_{\gT}(c)}$ and 
$$\rho_{min}^+(\gT)=\min_{c\in \gT} \mathbb{P}_{c^+,c_i^- \sim \rho^{k+1}} \left(c^+=c|Q=\gT, I^+ = \emptyset \\ \right)$$ 
In Theorem \ref{thm:multiclass_appdx} we will upper bound the following quantity: $\mathop{\E}\limits_{\gT \sim \gD}\bigg[\frac{\rho_{min}^+(\gT)}{p_{max}(\gT)}\  L_{sup}^\mu(\gT,\hat{f}) \bigg]$ ($\gD$ was defined in Section \ref{subsec:k-way_guarantees}).


\begin{theorem}\label{thm:multiclass_appdx}
Let $\hat{f}\in \argmin_{f\in \gF} \widehat{L}_{un}(f)$. With probability at least $1-\delta$, for all $f\in \gF$
\begin{equation*}\label{whole}
\mathop{\E}\limits_{\gT \sim \gD}\bigg[\frac{\rho_{min}^+(\gT)}{p_{max}(\gT)}\  L_{sup}^\mu(\gT,\hat{f}) \bigg]\leq \frac{1-\tau'}{1-\tau_k} L_{un}^{\neq}(f) + c'k\frac{\tau_1}{1-\tau_k}s(f)
+\frac{1}{1-\tau_k}Gen_{M} 
\end{equation*}
 where $c'$ is a constant.
\end{theorem}
Note that the definition of $s(f)$ used here is defined in Section \ref{sec:power_of_framework} 
\begin{proof}
First, we note that both hinge and logistic loss satisfy the following property:  $\forall I_1,I_2$ such that $I_1\cup I_2 = [t]$ we have that \begin{equation}\label{sets} \ell(\{\vv_i\}_{i\in I_1}) \leq \ell(\{\vv_i\}_{i\in[t]}) \leq \ell(\{\vv_i\}_{i\in I_1})+\ell(\{\vv_i\}_{i\in I_2})\end{equation}

We now prove the Theorem in 3 steps. First, we leverage the convexity of $\ell$ to upper bound a supervised-type loss with the unsupervised loss $L_{un}(f)$ of any $f\in \gF $. We call it supervised-type loss because it also includes degenerate tasks: $|\gT|=1$. Then, we decompose the supervised-type loss into an average loss over a distribution of supervised tasks, as defined in the Theorem, plus a degenerate/constant term. Finally, we upper bound the unsupervised loss $L_{un}(f)$ with two terms: $L_{un}^{\neq}(f)$ that measures how well $f$ contrasts points from different classes and an intraclass deviation penalty, corresponding to $s(f)$. 

{\bf \em Step 1 (convexity):} When the class $c$ is clear from context, we write $\hat{\mu}_{c}=\mathop{\E}\limits_{x\sim c}[\hat{f}(x)]$. Recall that the sampling procedure for unsupervised data is as follows: sample $c^+,c_1^-,...,c_k^-\sim \rho^{k+1}$ and then $x,x^+\sim \gD_{c^+}^2$ and $x_i^-\sim \gD_{c_i^-}$, $i\in[k]$. So, we have
\begin{equation}\label{Jen}
\begin{split}
L_{un}(\hat{f})&=\mathop{\E}\limits_{\substack{c^+,c_i^-\sim \rho^{k+1} \\ x,x^+ \sim  \gD_{c^+}^2 \\ x_i^- \sim \gD_{c_i^-}}} \left[ \ell \left(\left\{\hat{f}(x)^T\left(\hat{f}(x^+)-\hat{f}(x_i^-)\right)\right\}_{i=1}^k\right) \right] \\
 &= \mathop{\E}\limits_{\substack{c^+,c_i^-\sim \rho^{k+1}  \\ x\sim  \gD_{c^+} }}   \mathop{\E}\limits_{\substack{x^+\sim \gC_{c^+}\\x_i^- \sim \gD_{c_i^-}}}  \left[ \ell \left(\left\{\hat{f}(x)^T\left(\hat{f}(x^+)-\hat{f}(x_i^-)\right)\right\}_{i=1}^k\right) \right] 
\geq  \mathop{\E}\limits_{\substack{c^+,c_i^-\sim \rho^{k+1}  \\ x\sim  \gD_{c^+} }}    \left[ \ell \left(\left\{\hat{f}(x)^T\left(\hat{\mu}_{c^+}-\hat{\mu}_{c_i^-}\right)\right\}_{i=1}^k\right) \right]
\end{split}
\end{equation}
where the last inequality follows by applying the usual Jensen's inequality and the convexity of $\ell$. Note that in the upper bounded quantity, the $c^+,c_1^-,...,c_k^-$ don't have to be distinct and so the tuple does not necessarily form a task. 

{\bf \em Step 2 (decomposing into supervised tasks)}
We now decompose the above quantity to handle repeated classes.
\begin{align}\label{sup_dec}
\mathop{\E}\limits_{\substack{c^+,c_i^-\sim \rho^{k+1} \\ x\sim  \gD_{c^+} }}  &\left[ \ell \left(\left\{\hat{f}(x)^T\left(\hat{\mu}_{c^+}-\hat{\mu}_{c_i^-}\right)\right\}_{i=1}^k\right) \right]  \nonumber \\
& \geq (1-\tau_k) \mathop{\E}\limits_{\substack{c^+,c_i^-\sim \rho^{k+1} \\ x\sim  \gD_{c^+} }}    \left[ \ell \left(\left\{\hat{f}(x)^T\left(\hat{\mu}_{c^+}-\hat{\mu}_{c_i^-}\right)\right\}_{i=1}^k\right)  \Bigg| I^+= \emptyset\right] 
+\tau_k \mathop{\E}\limits_{c^+,c_i^-\sim \rho^{k+1}} [\ell(\underbrace{0,...,0}_\text{$|I^+|$ times}) | I^+ \neq \emptyset] \nonumber \\
&\geq  (1-\tau_k) \mathop{\E}\limits_{\substack{c^+,c_i^-\sim \rho^{k+1} \\ x\sim  \gD_{c^+} }}    \left[ \ell \left(\left\{\hat{f}(x)^T\left(\hat{\mu}_{c^+}-\hat{\mu}_{c}\right)\right\}_{\substack{c \in Q \\ c \neq c^+}} \right)  \Bigg| I^+= \emptyset\right] +\tau_k \mathop{\E}\limits_{c^+,c_i^-\sim \rho^{k+1} } \left[\ell_{|I^+|}(\vec{0})\ \Big{|} \ I^+ \neq \emptyset \right]
\end{align}
where $\ell_{t}(\vec{0})=\ell(0,\dots,0)$ ($t$ times).
Both inequalities follow from the LHS of \Eqref{sets}. Now we are closer to our goal of lower bounding an average supervised loss, since the first expectation in the RHS has a loss which is over a set of distinct classes. However, notice that this loss is for separating $c^+$ from $Q(c^+,c_1^-,...,c_k^-)\setminus \{c^+\}$. We now proceed to a symmetrization of this term to alleviate this issue.
\par Recall that in the main paper, sampling $\gT$ from $\gD$ is defined as sampling the (k+1)-tuple from $\rho^{k+1}$ conditioned on $I^+=\emptyset$ and setting $\gT=Q$. Based on this definition, by the tower property of expectation, we have

\begin{equation}\label{symmetrization1}
\begin{split}
&\mathop{\E}\limits_{\substack{c^+,c_i^-\sim \rho^{k+1} \\ x\sim  \gD_{c^+}}}    \left[ \ell \left(\left\{\hat{f}(x)^T\left(\hat{\mu}_{c^+}-\hat{\mu}_{c}\right)\right\}_{\substack{c \in Q \\ c \neq c^+}} \right) \Bigg | I^+= \emptyset\right]  \\
&= \mathop{\E}\limits_{\gT\sim \gD}  \mathop{\E}\limits_{\substack{c^+,c_i^-\sim \rho^{k+1} \\ x\sim  \gD_{c^+}}}  \Big[ \ell \Big(\Big\{\hat{f}(x)^T\big(\hat{\mu}_{c^+}-\hat{\mu}_{c}\big)\Big\}_{\substack{c \in Q \\ c \neq c^+}} \Big) \Big | Q= \gT,I^+= \emptyset\Big] \\
&= \mathop{\E}\limits_{\gT\sim \gD}  \mathop{\E}\limits_{\substack{c^+\sim \rho^+(\gT) \\ x\sim  \gD_{c^+} }}  \Big[ \ell \Big(\Big\{\hat{f}(x)^T\big(\hat{\mu}_{c^+}-\hat{\mu}_{c}\big)\Big\}_{\substack{c \in \gT  \\ c \neq c^+}} \Big) \Big]
\end{split}
\end{equation}
 where  $\rho^+(\gT)$ is the distribution of $c^+$ when $(c^+,c_1^-,...,c_k^-)$ are sampled from $\rho^{k+1}$ conditioned on $Q=\gT$ and $I^+=\emptyset$. Recall that $\rho_{min}^+(\gT)$ from the theorem's statement is exactly the minimum out of these $|\gT|$ probabilities. Now, to lower bound the last quantity with the LHS in the theorem statement, we just need to observe that for all tasks $\gT$
 
 \begin{equation}\label{symmetrization2}
\begin{split}
 &\mathop{\E}\limits_{\substack{c^+\sim \rho^+(\gT) \\ x\sim  \gD_{c^+} }}  \Big[ \ell \Big(\Big\{\hat{f}(x)^T\big(\hat{\mu}_{c^+}-\hat{\mu}_{c}\big)\Big\}_{\substack{c \in \gT  \\ c \neq c^+}} \Big) \Big] \\
 &\geq \frac{\rho_{min}^+(\gT)}{p_{max}(\gT)}\mathop{\E}\limits_{\substack{c^+\sim \gD_{\gT} \\ x\sim  \gD_{c^+} }}  \Big[ \ell \Big(\Big\{\hat{f}(x)^T\big(\hat{\mu}_{c^+}-\hat{\mu}_{c}\big)\Big\}_{\substack{c \in \gT  \\ c \neq c^+}} \Big) \Big]\\
 &= \frac{\rho_{min}^+(\gT)}{p_{max}(\gT)}L_{sup}(\gT,\hat{f})
\end{split}
\end{equation}
 By combining this with \Eqrefs{Jen}{sup_dec}{symmetrization2} we get 
 
 \begin{equation}\label{summary_sup}
 (1-\tau_k) \mathop{\E}\limits_{\gT\sim \gD} \Bigg[ \frac{\rho_{min}^+(T)}{p_{max}(T)} L_{sup}(\gT,\hat{f})  \Bigg]\leq L_{un}(\hat{f}) -\tau_k \mathop{\E}\limits_{c^+,c_i^-\sim \rho^{k+1} } \left[\ell_{|I^+|}(\vec{0})\ \Big|\  I^+ \neq \emptyset \right]
\end{equation}
Now, by applying Lemma \ref{generalization}, we bound the generalization error: with probability at least $1-\delta$, $\forall f\in \gF$
\begin{equation}\label{eq_generalize}
L_{un}(\hat{f})\leq L_{un}(f)+Gen_M
\end{equation}

However, $L_{un}(f)$ cannot be made arbitrarily small. One can see that for all $f\in \gF$, $L_{un}(f)$ is lower bounded by the second term in \Eqref{Jen}, which cannot be made arbitrarily small as $\tau_k > 0$.
\begin{equation}\label{lower bound}
L_{un}(f)\geq\mathop{\E}\limits_{\substack{c^+,c_i^-\sim \rho^{k+1} \\ x,x^+ \sim  \gD_{c^+} \\ x_i^- \sim \gD_{c_i^-}}} \left[ \ell \left(\left\{f(x)^T\left(f(x^+)-f(x_i^-)\right)\right\}_{i\in I^+}\right) \right] \geq \tau \mathop{\E}\limits_{c^+,c_i^-\sim \rho^{k+1} } \left[\ell_{|I^+|}(\vec{0}) \ \Big|\ I^+ \neq \emptyset \right] 
\end{equation}
where we applied Jensen's inequality. Since $\tau_k$ is not 0, the above quantity can never be arbitrarily close to 0 (no matter how rich $\gF$ is). 

{\em \bf Step 3 ($L_{un}$ decomposition)} 
Now, we decompose $L_{un}(f)$ by applying the RHS of \Eqref{sets}

\begin{align}
&\gL_{un}(f)\leq \mathop{\E}\limits_{\substack{c^+,c_i^- \sim \rho^{k+1} \\ x,x^+\sim \gD_{c^+}^2 \\ x_i^-\sim  \gD_{c_i^-}}} \Big[ \ell \Big( \Big\{f(x)^T\big(f(x^+)  -  f(x_i^-)  \big)\Big\}_{i\notin I^+ } 
\Big)
+ \ell \Big( \Big\{f(x)^T\big(f(x^+)  -  f(x_i^-)  \big)\Big\}_{i\in I^+} \Big)\Big]\\
&= \mathop{\E}\limits_{\substack{c^+,c_i^- \sim \rho^{k+1} \\ x,x^+\sim \gD_{c^+}^2 \\ x_i^-\sim  \gD_{c_i^-},\ i\notin I^+}} \Big[ \ell \Big( \Big\{f(x)^T\big(f(x^+)  -  f(x_i^-)  \big)\Big\}_{i\notin I^+ } 
\Big)\Big] 
+ \mathop{\E}\limits_{\substack{c^+,c_i^- \sim \rho^{k+1}  \\ x,x^+\sim \gD_{c^+}^2 \\ x_i^-\sim \gD_{c_i^-},\  i \in I^+ }}  \left[ \ell \left( \left\{f(x)^T\left(f(x^+)  -  f(x_i^-)  \right)\right\}_{i\in I^+ } \right)\right]\\
\begin{split}\label{unsup_dec}
= (1-\tau') \mathop{\E}\limits_{\substack{c^+,c_i^- \sim \rho^{k+1} \\ x,x^+\sim \gD_{c^+}^2  \\ x_i^- \sim  \gD_{c_i^-}, i \ \notin I^+ }} \Big[ \ell \Big( \Big\{f(x)^T\big(f(x^+)-  f(x_i^-)  \big)\Big\}_{i\notin I^+ } 
\Big)\Big{|} |I^+|<k\Big]\\
 \qquad  \qquad +\tau_k \mathop{\E}\limits_{\substack{c^+,c_i^- \sim \rho^{k+1}  \\ x,x^+\sim \gD_{c^+}^2 \\ x_i^-\sim \gD_{c_i^-},\  i \in I^+ }} \left[ \ell \left( \left\{f(x)^T\left(f(x^+)  -  f(x_i^-)  \right)\right\}_{i\in I^+ } \right)\Bigg{|}I^+\neq \emptyset \right]
\end{split}
\end{align}
Observe that the first term is exactly $(1-\tau')L_{un}^{\neq}(f)$. Thus, combining (\ref{summary_sup}), (\ref{eq_generalize}) and (\ref{unsup_dec}) we get

\begin{equation}\label{summary}
\begin{split}
 (1-\tau_k) \mathop{\E}\limits_{\gT\sim \gD} \Bigg[ \frac{\rho_{min}^+(T)}{p_{max}(T)}&L_{sup}(\gT,\hat{f})  \Bigg]\leq (1-\tau') L_{un}^{\neq}(f)+Gen_{M}\\
&+\tau_k \underbrace{\mathop{\E}\limits_{c^+,c_i^-\sim \rho^{k+1}} \Bigg[ \mathop{\E}\limits_{\substack{x,x^+ \sim \gD_{c^+}^2 \\ x_i^- \sim \gD_{c_i^-}, \ i\in  I^+}} \Big[
\ell \Big( \Big\{f(x)^T  \big(f(x^+)  -  f(x_i^-)  \big)\Big\}_{i\in I^+ }\Big)\Big]  - \ell_{|I^+|}(\vec{0}) \ \Bigg{|} I^+ \neq \emptyset \Bigg]}_{\Delta(f)}
\end{split}
\end{equation}
From the definition of $I^+$, $c_i^-=c^+$, $\forall i \in I^+$. Thus, from Lemma \ref{lemma:general_deviation}, we get that 

\begin{equation}\label{eq:Delta}
\Delta(f)\leq c'\mathop{\E}\limits_{c^+,c_i^-\sim \rho^{k+1}}\left[|I^+| \sqrt{\|\Sigma(f,c)\|_2}\mathop{\E}\limits_{x\sim \gD_c}[\|f(x)\|] \ \Big| \ I^+\neq \emptyset \right]
\end{equation}
for some constant $c'$.

Let $u$ be a distribution over classes with $u(c)=\mathbb{P}_{c^+,c_i^-\sim \rho^{k+1}}[c^+=c | I^+\neq \emptyset]$ and it is easy to see that $u(c)\propto \rho(c)\big(1-(1-\rho(c))^{k}\big)$ By applying the tower property to \Eqref{eq:Delta} we have
  \begin{equation}
\begin{split}
\Delta(f) \leq c'  \mathop{\E}\limits_{c \sim u }\left[ \mathop{\E}\limits_{ c^+,c_i^- \sim \rho^{k+1}}\left[|I^+| \big{|}c^+=c, I^+\neq \emptyset \right]\    \sqrt{ \| \Sigma(f,c)\|_2}\mathop{\E}\limits_{ x\sim \gD_{c} }\left[ \|f(x)\| \right]   \right]
\end{split}
\end{equation}
 But, 
\begin{equation}
\begin{split}
 \mathop{\E}\limits_{ c^+,c_i^- \sim \rho^{k+1}}\big[|I^+| \big{|}c^+=c, I^+\neq \emptyset \big]&=\sum_{i=1}^k \mathbb{P}_{c^+,c_i^- \sim \rho^{k+1}}\big(c_i^-=c^+ \big{|} c^+=c,I^+\neq \emptyset\big) \\
& =k\mathbb{P}_{c^+,c_i^- \sim \rho^{k+1}}\big(c_1^-=c^+ \big{|} c^+=c,I^+\neq \emptyset\big) \\
& =k\frac{\mathbb{P}_{c^+,c_i^- \sim \rho^{k+1}}\big(c_1^-=c^+=c \big)}{\mathbb{P}_{c^+,c_i^- \sim \rho^{k+1}}\big(c^+=c, I^+\neq \emptyset\big)}\\
&=k\frac{\rho^2(c)}{\rho(c)\big(1-(1-\rho(c))^{k}\big)}=k\frac{\rho(c)}{1-(1-\rho(c))^{k}}
\end{split}
\end{equation}
Now, using the fact that $\tau_k=1-\sum_{c'}\rho(c')(1-\rho(c'))^{k}=\sum_{c'}\rho(c')\left(1-(1-\rho(c'))^{k}\right)$ and $\tau_1=\sum_c\rho^2(c)$,
\begin{equation}
\begin{split}
\frac{\tau_k}{1-\tau_k}\Delta(f)&\leq \frac{\tau_k}{1-\tau_k} c'  \mathop{\E}\limits_{c \sim u }\left[k\frac{\rho(c)}{1-(1-\rho(c))^{k}} \sqrt{ \| \Sigma(f,c)\|_2}\mathop{\E}\limits_{ x\sim \gD_{c} }\left[ \|f(x)\|\right] \right]\\
&=c'k\frac{\tau_k}{1-\tau_k} \sum_c \frac{\rho^2(c)}{\sum_{c'}\rho(c')\left(1-(1-\rho(c'))^{k}\right)} \sqrt{ \| \Sigma(f,c)\|_2}\mathop{\E}\limits_{ x\sim \gD_{c} }\left[ \|f(x)\|\right] \\
&=c'k\frac{\tau_1}{1-\tau_k} \mathop{\E}\limits_{c \sim \nu }\left[ \sqrt{ \| \Sigma(f,c)\|_2}\mathop{\E}\limits_{ x\sim \gD_{c} }\left[ \|f(x)\|\right] \right]=c'k\frac{\tau_1}{1-\tau_k}s(f)
\end{split}
\end{equation}
and we are done.
\end{proof}

\subsection{Competitive Bound}\label{appdx:subgaussian_multiclass}
As in Section \ref{subsec:subgaussian}, we prove a competitive type of bound, under similar assumptions. Let $\ell_\gamma(\vv)=\max\{0,1+\max_i\{-\vv_i\}/\gamma\}$, $\vv\in \mathbb{R}^k$, be the multiclass hinge loss with margin $\gamma$ and for any $\gT$ let $L_{\gamma,sup}^\mu(\gT,f)$ be  $L_{sup}^\mu(\gT,f)$ when $\ell_\gamma$ is used as loss function. For all tasks $\gT$, let ${\rho'}^+(\gT)$ is the distribution of $c^+$ when $(c^+,c_1^-,...,c_k^-)$ are sampled from $\rho^{k+1}$ conditioned on $Q=\gT$ and $|I^+|<k$. Also, let
${\rho'}_{max}^+(\gT)$ be the maximum of these $|\gT|$ probabilities and $p_{min}(\gT)=\min_{c\in \gT}\gD_{\gT}(c)$. 

We will show a competitive bound against the following quantity, for all $f\in \gF$: $\mathop{\E}\limits_{\gT\sim \gD'}\left[\frac{{\rho'}_{max}^+(\gT)}{p_{min}(\gT)}\gL_{\gamma,sup}^\mu(\gT,f)\right]$, where $\gD'$ is defined as follows: sample $c^+,c_1^-,...,c_k^- \sim \rho^{k+1}$, conditioned on $|I^+|<k$. Then, set $\gT=Q$. Observe that when $I^+=\emptyset$ with high probability, we have $\gD'\approx \gD$.
\begin{lemma}\label{lemma:multiclass_subgaussian}
For all $f \in \gF$ suppose the random variable $f(X)$, where $X\sim D_c$, is $\sigma^2(f)$-subgaussian in every direction for every class $c$ and has maximum norm $R(f) = max_{x\in \gX} \|f(x)\|$. Let $\widehat{f}\in \argmin_{f\in \gF} \widehat{L}_{un}(f)$. Then for all $\epsilon > 0$, with probability at least $1-\delta$, for all $f\in \gF$
\begin{align*}
\mathop{\E}\limits_{\gT \sim \gD}\bigg[\frac{\rho_{min}^+(\gT)}{p_{max}(\gT)}\  L_{sup}^\mu(\gT,\hat{f}) \bigg] \leq  \alpha \gamma(f) \mathop{\E}\limits_{\gT\sim \gD'}\left[\frac{{\rho'}_{max}^+(\gT)}{p_{min}(\gT)}\gL_{\gamma,sup}^\mu(\gT,f)\right] + \beta s(f)+ \eta Gen_{M} +\epsilon
\end{align*}
where $\gamma(f)=1+c'R(f)\sigma(f) (\sqrt{\log{k}}+\sqrt{\log {\frac{R(f)}{\epsilon}} })$, $c'$ is some constant, $\alpha=\frac{1-\tau'}{1-\tau_k}$, $\beta=k\frac{\tau_1}{1-\tau_k}$ and $\eta=\frac{1}{1-\tau_k}$.
\end{lemma} 
\begin{proof}
We will show that $\forall f\in \gF$
\begin{equation}
L_{un}^{\neq} (f)\leq \gamma(f) \mathop{\E}\limits_{\gT\sim \gD'}\left[\frac{{\rho'}_{max}^+(\gT)}{p_{min}(\gT)}\gL_{\gamma,sup}^\mu(\gT,f)\right]
\end{equation}
and the Lemma follows from Theorem \ref{thm:multiclass_theorem}. Now, we fix an $\epsilon>0$, an $f\in \gF$ and we drop most of the arguments $f$ in the rest of the proof. Also, fix $c^+,c_1^-\dots c_k^-, x$ and let $t=k-|I^+|$. We assume without loss of generality, that $c^+\neq c_i^-$, $\forall i\in[t]$. Now, 
\begin{equation}
\max_{i\in [t]}f(x)^T(f(x_i^-)-f(x^+))\leq \mu+\max_i z_i^- -z^+
\end{equation}
where $\mu=\max_{i\in [t]}f(x)^T(\mu_{c_i^-}-\mu_{c^+})$, $z_i^-=f(x)^T(f(x_i^-)-\mu_{c_i^-})$ and $z^+=f(x)^T(f(x^+)-\mu_{c^+})$. $z_i$ are centered $\sigma^2R^2$-subgaussian, so from standard properties of subgaussian random variables $\mathbb{P}[\max_iz_i^-\geq \sqrt{2}\sigma R\sqrt{\log{t}}+\sqrt{2c_1}\sigma R \sqrt{\log{R/\epsilon}}]\leq(\epsilon/R)^{c_1}$
(again we consider here the case where $R\geq 1$ and for $R<1$, the same arguments hold but with removing $R$ from the $\log$). $z^+$ is also centered $\sigma^2R^2$-subgaussian, so  $\mathbb{P}[z^+\geq \sqrt{2c_1}\sigma R \sqrt{\log{R/\epsilon}}]\leq(\epsilon/R)^{c_1}$. Let $\gamma=1+c'\sigma R(\sqrt{\log{t}}+ \sqrt{\log{R/\epsilon}})$ for appropriate constant $c'$. By union bound, we have $p=\mathbb{P}[\max_i z_i^- -z^+\geq \gamma-1]\leq 2(\epsilon/R)^{c_1}$. Thus, $\mathbb{E}_{z^+,z_i^-}[(1+\mu+\max_iz_i^- -z^+)_+]\leq (1-p)(\mu+\gamma)_+ +p (2R^2+1)\leq \gamma(1+\mu/\gamma)_+ + \epsilon$ (for appropriate constant $c_1$). By taking expectation over $c^+,c_i^-\sim \rho^{k+1}$, conditioned on $|I^+|<k$ , and over $x\sim \gD_{c^+}$ we get

\begin{equation}
\begin{split}
L_{un}^{\neq} (f)& \leq \gamma \mathop{\E}\limits_{\substack {c^+,c_i^-\sim \rho^{k+1}\\ x\sim \gD_{c^+}}}\left[\left(1+\frac{\max_{c\in Q, c\neq c^+} f(x)^T(\mu_c-\mu_{c^+})}{\gamma}\right)_+\bigg{|} |I^+|<k\right]\\
&=\gamma \mathop{\E}\limits_{\gT\sim \gD'}\mathop{\E}\limits_{\substack {c^+,c_i^-\sim \rho^{k+1}\\ x\sim \gD_{c^+}}}\left[\left(1+\frac{\max_{c\in Q, c\neq c^+} f(x)^T(\mu_c-\mu_{c^+})}{\gamma}\right)_+\bigg{|} Q=\gT, |I^+|<k\right]\\
&= \gamma \mathop{\E}\limits_{\gT\sim \gD'}\mathop{\E}\limits_{\substack{c^+\sim{\rho'}^+(\gT)\\ x\sim \gD_{c^+}} }\left[\left(1+\frac{\max_{c\in T, c\neq c^+} f(x)^T(\mu_c-\mu_{c^+})}{\gamma}\right)_+\right]\leq  \gamma \mathop{\E}\limits_{\gT\sim \gD'}\left[\frac{{\rho'}_{max}^+(\gT)}{p_{min}(\gT)}\gL_{\gamma,sup}^\mu(\gT,f)\right]
\end{split}
\end{equation}
\end{proof}
%

\section{Examples for Section \ref{subsec:k-way_effect}}\label{appdx:multiclass_examples}
Here, we illustrate via examples two ways in which the increase of $k$ can lead to suboptimal $\hat{f}$. We will consider the hinge loss as the loss function, while the examples carry over trivially for logistic loss.

\begin{enumerate}

\item The first example is the case where even though there exist representations in $\gF$ that can separate every class, the suboptimal representation is picked by the algorithm when $k = \Omega(|\gC|)$. Let $\gC = \{c_i\}_{i\in [n]}$ where for each class, $D_{c_i}$ is uniform over two points $\{x^1_{i}, x^2_{i}\}$. Let $e_i$ be the indicator vectors in $\R^n$ and let the class $\gF$ consists of $\{ f_0, f_1\}$ with $f_0, f_1: \gX \mapsto \R^n $ where $f_1(x^1_i) = 3/2 r e_i$ and $f_1(x^2_i) = 1/2 r e_i$ for all $i$, for some $r>0$, and $f_0 = \vec{0}$. Finally, $\rho$ is uniform over $\gC$. Now, when the number of negative samples is $\Omega(n)$, the probability that $\exists j \in [k]$ such that $c^+ = c^-_j$ is constant, and therefore $L_{un}(f)=\Omega(r^2)>1=L_{un}(f_0)$ when $r$ is large. This means that despite $L_{sup}(\gC,f_1) = 0$, the algorithm will pick $f_0$ which is a suboptimal representation. 


\item We can extend the first example to the case where, even when $k = o(|\gC|)$, the algorithm picks suboptimal representations. To do so, we simply `replicate' the first example to create clusters of classes. Formally, let $\gC = \{c_{ij}\}_{i, j \in [n]}$ where for each class, $D_{c_{ij}}$ is uniform over two points $\{ x^1_{ij}, x^2_{ij} \}$. Finally, same as above, let $\gF$ consist of two functions $\{f_0, f_1 \}$. The function $f_1$ maps $f_1(x^1_{ij}) = 3/2 r e_i$ and $f_1(x^2_{ij}) = 1/2 r e_i$ for all $i, j$ and $f_0=\vec{0}$. $\rho$ is uniform over $\gC$. Now, note that $f_1$ `clutsters' the $n^2$ classes and their points into $n$ clusters, each along an $e_i$. Thus, it is only useful for contrasting classes from different clusters. However, note that the probability of intra-cluster collision with $k$ negative samples is $1-(1-1/n)^{k}$. When $k = o(n)$, we have that $L_{un}(f_1) = o(1) < 1 = L_{un}(f_0)$ so the algorithm will pick $f_1$. However, when $k = \Omega(n)$, $L_{un}(f) = \Omega(r^2) > 1 = L_{un}(f_0)$ and the algorithm will pick the suboptimal representation $f_0$. Thus, despite $|\gC| = n^2$, having more than $n$ negative samples can hurt performance, since even tough $f_1$ cannot solve all the tasks, the average supervised loss over $t$-way tasks, $t=o(n)$, is $L_{sup}(f)\leq O(1-(1-1/n)^{t-1})=o(1)$.
%
%
\end{enumerate}










\section{Experiments}
\subsection{Wiki-3029 construction}
We use the Wikipedia dump and select articles that have entries in the WordNet, have at least 8 sections and at least 12 sentences of length at least 4 per section.
At the end of this filtering we are left with 3029 articles with at least 200 sentences per article.
We then sample 200 sentences from each article and do a 70\%/10\%/20\% train/dev/test split.

\subsection{GRU model}
We use a bi-directional GRU with output dimension of 300 trained using dropout 0.3. The input word embeddings are initialized to pretrained CC GloVe vectors and fixed throughout training.

\end{document}